\documentclass[10pt,twocolumn,letterpaper]{article}

\usepackage{cvpr}              
\usepackage[accsupp]{axessibility} 

\usepackage{amsthm}
\usepackage{thm-restate}

\theoremstyle{plain}
\newtheorem{theorem}{Theorem}[section]

\theoremstyle{definition}

\theoremstyle{remark}
\newtheorem{remark}[theorem]{Remark}

\usepackage{colortbl}
\usepackage{array}
\usepackage[table]{xcolor} 

\usepackage{bbding}
\newcommand{\yes}{\Checkmark}
\newcommand{\no}{\XSolidBrush}

\renewcommand{\paragraph}[1]{\vspace{.5em}\noindent\textbf{#1\ \ }}

\usepackage{multirow}

\usepackage{amsmath,amsfonts,bm}

\def\eqref#1{equation~\ref{#1}}

\def\1{\bm{1}}

\def\vzero{{\bm{0}}}
\def\vone{{\bm{1}}}

\def\vb{{\bm{b}}}

\def\ve{{\bm{e}}}

\def\vg{{\bm{g}}}

\def\vl{{\bm{l}}}

\def\vp{{\bm{p}}}

\def\vt{{\bm{t}}}

\def\vv{{\bm{v}}}
\def\vw{{\bm{w}}}
\def\vx{{\bm{x}}}

\def\vz{{\bm{z}}}

\def\mM{{\bm{M}}}

\def\mT{{\bm{T}}}

\DeclareMathAlphabet{\mathsfit}{\encodingdefault}{\sfdefault}{m}{sl}
\SetMathAlphabet{\mathsfit}{bold}{\encodingdefault}{\sfdefault}{bx}{n}

\def\sM{{\mathbb{M}}}

\def\sS{{\mathbb{S}}}

\newcommand{\E}{\mathbb{E}}

\newcommand{\R}{\mathbb{R}}

\newcommand{\softmax}{\mathrm{softmax}}

\DeclareMathOperator*{\argmax}{arg\,max}

\DeclareMathOperator{\diag}{diag}

\newcommand{\algname}{\textit{Ramen}}

\let\ab\allowbreak

\ab

\newcommand{\normalize}{\mathrm{normalize}}

\newcommand{\meansd}[2]{#1 \scriptsize(#2)}

\definecolor{cvprblue}{rgb}{0.21,0.49,0.74}
\usepackage[pagebackref,breaklinks,colorlinks,allcolors=cvprblue]{hyperref}

\def\paperID{42152} 
\def\confName{CVPR}
\def\confYear{2026}

\title{{\algname}: Robust Test-Time Adaptation of Vision-Language Models with \\ Active Sample Selection}

\author{Wenxuan Bao\thanks{Equal contribution.} \quad Yanjun Zhao\footnotemark[1] \quad Xiyuan Yang \quad  Jingrui He\thanks{Corresponding author.}\\
University of Illinois Urbana-Champaign\\
{\tt\small \{wbao4,yanjunzh,xiyuany4,jingrui\}@illinois.edu}
}

\begin{document}
\maketitle

\addtocontents{toc}{\protect\setcounter{tocdepth}{0}} 

\begin{abstract}
    Pretrained vision-language models such as CLIP exhibit strong zero-shot generalization but remain sensitive to distribution shifts. Test-time adaptation adapts models during inference without access to source data or target labels, offering a practical way to handle such shifts. However, existing methods typically assume that test samples come from a single, consistent domain, while in practice, test data often include samples from mixed domains with distinct characteristics. Consequently, their performance degrades under mixed-domain settings. 
    To address this, we present {\algname}, a framework for robust test-time adaptation through active sample selection. For each incoming test sample, {\algname} retrieves a customized batch of relevant samples from previously seen data based on two criteria: domain consistency, which ensures that adaptation focuses on data from similar domains, and prediction balance, which mitigates adaptation bias caused by skewed predictions.
    To improve efficiency, {\algname} employs an embedding-gradient cache that stores the embeddings and sample-level gradients of past test images. The stored embeddings are used to retrieve relevant samples, and the corresponding gradients are aggregated for model updates, eliminating the need for any additional forward or backward passes. 
    Our theoretical analysis provides insight into why the proposed adaptation mechanism is effective under mixed-domain shifts. 
    Experiments on multiple image corruption and domain-shift benchmarks demonstrate that {\algname} achieves strong and consistent performance, offering robust and efficient adaptation in complex mixed-domain scenarios. 
    Our code is available at \texttt{\url{https://github.com/baowenxuan/Ramen}}. 
\end{abstract}
    
\section{Introduction}
\label{sec:introduction}

Pretrained vision-language models (VLMs) such as CLIP~\cite{clip} have demonstrated strong zero-shot generalization across a wide range of vision tasks~\cite{regionclip,denseclip,styleclip}. However, their performance can degrade significantly under distribution shifts, including image corruptions~\cite{corruption} and domain shifts~\cite{domainbed}. Test-time adaptation (TTA) has emerged as a promising strategy for improving model robustness under distribution shifts, by adapting the model during test-time without accessing source data or target labels~\cite{tta_survey1,tta_survey2,tta_survey3}. This property makes TTA particularly suitable for the adaptation of pretrained VLMs, where the source training data is often large-scale, proprietary, or unavailable at deployment. 

Although existing TTA methods have demonstrated impressive performance on standard benchmarks, most of them rely on a key assumption that testing samples are drawn from a single, consistent domain. In practice, however, this assumption is often violated, as test data may contain samples from multiple domains with distinct characteristics~\cite{rotta,sar,sar2}. For instance, a user’s photo gallery on a smartphone may include images captured under different weather conditions and lighting environments, or even collected from different platforms. In such \textbf{mixed-domain scenarios}, conventional TTA algorithms often exhibit degraded performance, as they struggle to generalize across diverse and inconsistent test distributions~\cite{eata,sar,sar2,roid,unmix_tns,vte}.

We argue that this degradation arises because a single model is forced to adapt simultaneously to data from multiple, distinct domains. As a result, the model cannot specialize for each domain but instead adapts to an averaged domain representation, leading to suboptimal adaptation.
To address this issue, we propose {\algname}, which enables \textbf{R}obust test-time \textbf{a}daptation with sa\textbf{m}ple s\textbf{e}lectio\textbf{n}. Instead of passively adapting on the entire mixed test stream, which dilutes domain-specific signals, {\algname} actively constructs a \textbf{customized support set}, a small batch of previously seen test samples most relevant to the current one, for model adaptation. 
The selected samples are determined by two criteria: (1) \textbf{domain consistency}, which ensures that adaptation focuses on data from similar domains measured by their distance in image embeddings, and (2) \textbf{prediction balance}, which mitigates adaptation bias caused by skewed model predictions. 
To further reduce the computational cost of per-sample adaptation, {\algname} employs an \textbf{embedding-gradient cache} that stores both the embeddings and sample-level gradients of past test images. The stored embeddings are used for sample retrieval, and the corresponding gradients are aggregated for model updates, eliminating the need for additional forward or backward passes. 
Our theoretical analysis provides insight into why {\algname} is effective under mixed-domain shifts, and extensive experiments across image corruption and domain-shift benchmarks demonstrate that {\algname} achieves robust and efficient adaptation in complex mixed-domain scenarios.
\textbf{We summarize our main contributions as follows:}
\begin{itemize}
    \item We introduce an \textit{active sample selection} framework to enable effective adaptation under mixed-domain shifts, guided by two empirical selection criteria: \textit{domain consistency} and \textit{prediction balance}.
    \item We propose an efficient algorithm, {\algname}, which implements active sample selection through an \textit{embedding-gradient cache} that reuses stored embeddings and gradients for lightweight adaptation.
    \item We provide theoretical analysis that reveals the underlying principles of a broad class of TTA methods and explains why {\algname} improves adaptation performance under mixed-domain settings.
    \item We conduct extensive experiments on image corruption and domain-shift benchmarks, demonstrating both the \textit{effectiveness} and \textit{efficiency} of {\algname}. 
\end{itemize}

\section{Related Works}
\label{sec:related_works}

In this section, we summarize related works of TTA in both single-domain and mixed-domain scenarios. We provide a broader discussion of related works in Appendix \ref{appendix:discussion:related_work}. 

\paragraph{Single-domain TTA}
TTA adapts a pre-trained model to an unlabeled target domain without access to source data. Single-domain TTA assumes that all target samples are drawn from one consistent domain. Among existing methods, one common approach optimizes a self-supervised objective, such as entropy minimization~\cite{tent,note,eata,deyo}, image–text alignment~\cite{clipartt,watt}, or inter-class variance~\cite{mint}, to adjust the model’s normalization layers in response to distribution shifts. Another category is memory-based, which stores embeddings of high-confidence samples and uses embedding similarity to refine model predictions~\cite{adanpc,dmn,tda}. Augmentation-based methods generate multiple augmented views for each image and either aggregate predictions across views~\cite{vte,zero} or minimize marginal entropy to enforce consistency~\cite{tpt,tps}, but these approaches often incur substantially higher computational costs. 

\paragraph{Mixed-domain TTA and more}
Mixed-domain TTA refers to the setting where a model must adapt to a target data stream containing samples from multiple domains. Prior works such as SAR~\cite{sar} and ROID~\cite{roid} revealed the performance drop of single-domain TTA methods in this setting and proposed techniques like sharpness-aware minimization and weight ensembling to mitigate model collapse. However, these approaches still rely on a single model adapting simultaneously to diverse domains, fundamentally limiting their effectiveness. 
UnMix-TNS~\cite{unmix_tns} addressed this issue by modifying BatchNorm \cite{batchnorm} to maintain multiple sets of running statistics, but this solution applied only to BatchNorm layers, which are generally discouraged under mixed-domain settings \cite{sar}. 
A related yet distinct scenario is \textit{continual TTA}~\cite{cotta,rotta}, where the model is sequentially adapted to a series of domains, one at a time. The key difference lies in the data continuity: in continual TTA, consecutive samples typically come from the same domain, whereas in mixed-domain TTA, even samples within a single batch may originate from different domains. 

\section{Proposed Method: {\algname}}
\label{sec:method}

\begin{figure*}
    \centering
    \includegraphics[width=0.85\linewidth]{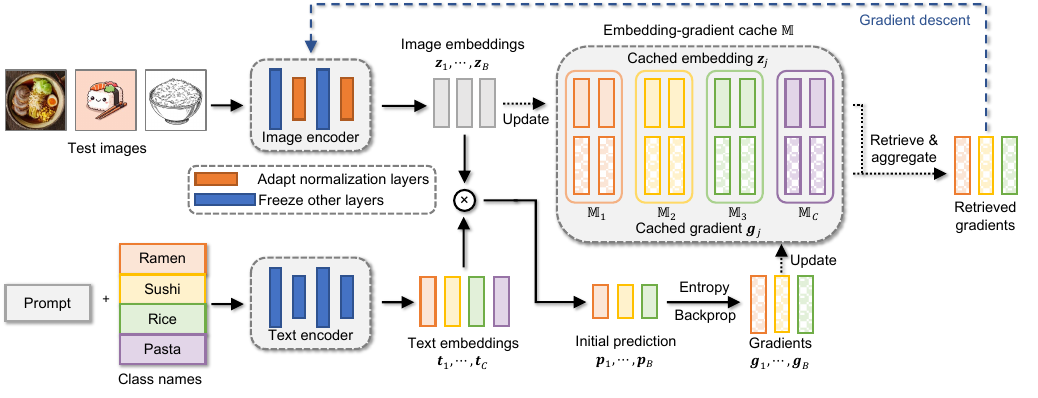}
    \caption{Overview of {\algname}. For each test sample, (1) compute its image embedding, pseudo-label, and sample-level gradient; (2) update the class-specific memory with these entries; (3) retrieve a support set from the memory; (4) aggregate the cached gradients for model update; and (5) perform inference and reset the model parameters.}
    \label{fig:overview}
\end{figure*}

In this section, we introduce our proposed method, {\algname}.
Subsection \ref{subsec:method:prelim} formally defines the research problem and the key challenges.
Subsection \ref{subsec:method:select} presents the active sample selection mechanism and its selection criteria.
Subsection \ref{subsec:method:eg_cache} explains how we leverage an embedding–gradient cache to improve the efficiency of active sample selection, and summarizes the overall algorithm.
Figure \ref{fig:overview} gives an overview of our proposed method. 


\subsection{Preliminary}
\label{subsec:method:prelim}

\textbf{CLIP}~\cite{clip} is a vision–language model composed of an image encoder and a text encoder, jointly trained to align visual and textual representations. Leveraging large-scale image–caption pairs, CLIP learns generalizable representations and demonstrates strong zero-shot recognition capability. For a classification task with $C$ categories, the text encoder converts each class description (e.g., ``\texttt{a photo of a \{class\}}'') into normalized text embeddings $\mT = [\vt_1, \cdots, \vt_C]^\top \in \R^{C \times d}$, where $d$ denotes the embedding dimension. Given a test image $\vx_i$, the image encoder outputs a normalized embedding $\vz_i \in \R^d$. The predicted probability vector is $\vp_i = \softmax(\exp(t) \cdot \vz_i \mT^\top)$, where $\exp(t)$ is a temperature parameter. The predicted label corresponds to the class with highest probability, i.e., $\hat{y}_i = \argmax_y p_{iy}$. Despite its strong generalization ability, CLIP’s performance drops substantially when facing distribution shifts~\cite{watt,clipartt,mint}, such as image corruptions~\cite{corruption} or domain shifts~\cite{domainbed}.

\paragraph{Entropy minimization}
Many existing TTA methods minimize entropy as a surrogate loss~\cite{tent,eata,sar,deyo}. Given CLIP’s predicted probability vector $\vp_i = [p_{i1}, \cdots, p_{iC}]^\top$, the entropy for sample $i$ is computed as $H(\vx_i) = - \sum_{c=1}^C p_{ic} \log p_{ic}$, which measures the model's uncertainty on the current sample. The model parameters are then updated by minimizing the average entropy over a batch or in an online (streaming) manner.
In practice, most methods only update the normalization layers rather than the entire model.
Such an update is highly parameter-efficient (e.g., for the ViT-B/16 visual encoder, fewer than 0.05\% of parameters are involved) and helps mitigate catastrophic forgetting during continuous adaptation. Our proposed method follows this general paradigm, as it also minimizes the entropy loss while adapting only the normalization layers.

\paragraph{Challenge of mixed-domain shift}
During the testing stage, the model performs on-the-fly adaptation and evaluation over a stream of data batches.
Conventional TTA methods typically assume that the test data are drawn from a single, consistent domain.
However, in real-world scenarios, the test data may come from a mixture of domains.
It has been observed that under such mixed-domain shifts, the effectiveness of TTA drops significantly \cite{eata,sar,sar2,roid,unmix_tns,vte}, which is also confirmed by our experiments (see Figure \ref{fig:single_vs_mixed}).
We argue that this degradation arises because a single model is forced to adapt simultaneously to data from multiple, distinct domains.
For example, in a single-domain setting, data from different domains are encountered separately, allowing sufficient adaptation to each domain.
In contrast, under mixed-domain settings, even samples within the same batch may originate from different domains, yet they are forced to share one common model with identical parameters.
As a result, the model cannot specialize for each domain. Instead, it adapts to an averaged representation across domains, leading to suboptimal adaptation.


\subsection{Active Sample Selection}
\label{subsec:method:select}

To address this issue, we assign each test sample $\vx_i$ a distinct model weight $\vw_i$. Each $\vw_i$ is obtained by the adapting the model on a sample-specific \textit{support set} $\sS_i$, i.e., a subset of previously seen data for adaptation. Formally, this process can be expressed as
\begin{align}
    \hat{y}_i = \mathrm{CLIP}(\vx_i; \vw_i), \quad \text{where } \vw_i = \mathrm{TTA}(\vw; \sS_i),
\end{align}
where $\vw$ denotes the weights of the pre-trained model, and $\mathrm{TTA}()$ represents the test-time adaptation process, which takes the model weights and a support set as input and outputs the adapted weights. $\mathrm{CLIP}()$ refers to the CLIP prediction process as described in Subsection \ref{subsec:method:prelim}. 

In this subsection, we introduce how to select the support set $\sS_i$ for each test sample $\vx_i$. We propose two selection criteria: domain consistency and prediction balance. 
\begin{itemize}
    \item \textbf{Domain consistency} requires that the samples in $\sS_i$ should come from the same or similar domain as $\vx_i$. Although the domain label of each sample and the total number of domains are unknown, pretrained VLMs are typically trained for general-purpose understanding, and their image embeddings implicitly contain domain-related information. Therefore, the more similar two image embeddings are, the more likely they originate from the same domain. We provide an empirical validation in Appendix \ref{appendix:discussion:empirical_validation}.
    \item \textbf{Prediction balance} ensures that the samples in $\sS_i$ have balanced predictions across classes. If $\sS_i$ is dominated by samples predicted as a single class, the adaptation process may introduce a prediction bias, making the model more likely to assign future samples to that class \cite{sar,bot}. Maintaining prediction balance helps prevent such bias and improves adaptation stability.
\end{itemize}

To realize active sample selection that satisfies both domain consistency and prediction balance, we maintain a class-split memory $\sM$ to store information of previously seen test samples. These stored samples are later used for model adaptation (the detailed contents of the memory will be discussed in Subsection~\ref{subsec:method:eg_cache}).
The memory $\sM$ consists of $C$ first-in-first-out (FIFO) queues $\sM_1, \cdots, \sM_C$, one for each class. Each queue $\sM_c$ stores up to $K$ most recently observed samples that are predicted as class $c$ by the zero-shot classifier, where $K$ is the maximum capacity per queue. 

When a new test sample $\vx_i$ arrives, we use its image embedding $\vz_i$ to retrieve the top-$k$ most similar samples $\sS_{ic}$ from each class queue. This design \textit{ensures domain consistency}, as retrieved samples are close to $\vx_i$ in the embedding space. Then, by concatenating an equal number of samples from each class queue, we form the final support set $\sS_i$, which \textit{inherently satisfies prediction balance}. Formally, this process can be expressed as
\begin{align}
    \sS_i = \bigcup_{c=1}^C \sS_{ic}, \quad \text{where } \sS_{ic} = \underset{j \in \sM_c}{\mathrm{Top-}k}\left(\vz_i^\top \vz_j \right). 
    \label{eq:select}
\end{align}


\subsection{Embedding-Gradient Cache}
\label{subsec:method:eg_cache}

While the proposed active sample selection enables customized adaptation for each sample, a naive implementation can be computationally prohibitive. In standard TTA, each test sample requires one backward pass for adaptation.\footnote{Following~\cite{eata}, when computing the gradient over a batch of $B$ samples, we count it as $B$ backward passes, since the overall computation cost scales linearly with $B$.} However, if we directly recompute gradients on the support set $\sS_i$, each test sample $\vx_i$ requires $C\cdot k$ backward passes, inflating the overall computational cost by a factor of $C\cdot k$. Such a multiplicative increase makes a straightforward implementation impractical for online or real-time inference. 

To address this issue, we propose an efficient mechanism called \textbf{embedding-gradient cache}, which leverages the point-wise nature of common TTA objectives such as the entropy loss. For the entropy loss, the loss over a batch $\sS_i$ can be written as a weighted average of the sample-level losses: 
\begin{align}
    H(\sS_i) =  \sum_{j \in \sS_i} \alpha_{ij} H(\vx_j), 
\end{align}
where $\alpha_{ij}$ denotes the weight of sample $j$. Consequently, the gradient over the batch also satisfies
\begin{align}
    \nabla_{\vw} H(\sS_i) = \sum_{j \in \sS_i} \alpha_{ij} \vg_j, \quad \text{where } \vg_j = \nabla_{\vw} H(\vx_j). 
    \label{eq:grad}
\end{align}
Therefore, we can cache the sample-level gradients $\vg_i = \nabla_{\vw} H(\vx_j)$ in memory.
In this way, we no longer need to recompute gradients over the support set $\sS_j$; instead, we can simply aggregate the cached gradients to obtain the batch-level gradient efficiently. 

Therefore, we adopt an embedding–gradient cache mechanism.
In the memory, we store the image embeddings $\vz_j$ and the corresponding sample-level gradients $\vg_j$ of previously seen samples. For each incoming sample $i$, we perform the following steps:
\begin{enumerate}
    \item Compute its image embedding $\vz_i$, pseudo-label $\hat{y}_i$, and sample-level gradient $\vg_i$;
    \item Update memory $\sM_{\hat{y}_i}$ with the new entries $(\vz_i, \vg_i)$; 
    \item Retrieve the corresponding support set $\sS_i$ using Eq.~(\ref{eq:select});
    \item Aggregate the cached gradients according to Eq.~(\ref{eq:grad}) and update the model;
    \item Perform inference with the updated model, and then reset the model parameters.
\end{enumerate}

For the aggregation weights $\alpha_{ij}$, we employ two common strategies in TTA: entropy weighting \cite{eata,sar,deyo} and similarity weighting \cite{tda,dmn,latte}:
\begin{align}
    \alpha_{ij} = \exp(- H(\vx_j)) \cdot \exp(- \beta \cdot \| \vz_i - \vz_j \|_2). 
\end{align}
where $H(\vx_j)$ denotes the entropy of sample $j$, and $\beta$ is a hyper-parameter controlling the influence of similarity. 
Entropy weighting assigns larger weights to samples with lower entropy (i.e., higher prediction confidence), as they are considered more reliable. 
Similarity weighting assigns larger weights to samples that are closer to the target sample in the embedding space, since they are more likely to belong to the same domain. 

Although the above discussion focuses on a single test sample, in practice we use a simple reparameterization trick to parallelize sample-level gradient computation within a batch, significantly improving efficiency. Implementation details are provided in Appendix~\ref{appendix:discussion:reparam}. 

\section{Theoretical Analysis}

In this section, we analyze the underlying principles of TTA methods that update normalization layers by minimizing prediction entropy, and explain why active sample selection can further improve their performance.

\paragraph{Setup and assumptions}
Most TTA methods \cite{tent,note,sar,watt,clipartt,mint} adapt models by updating the affine parameters in normalization layers. Therefore, we focus our analysis on this component. Common normalization layers \cite{batchnorm,groupnorm,layernorm,rmsnorm} share a same structure: a normalization step followed by an element-wise affine/linear transformation. Although the normalization step differs across layer types, the parameterization of the affine/linear transformation remains similar. Following \cite{mint}, we consider a single normalization layer in a binary classification setting. Let $\vv_i \in \R^d$ denote the intermediate normalized feature, obtained after the normalization step but before applying its affine transformation. The final image embedding is given by an element-wise linear transformation: 
\begin{align}
    \vz_i = \vv_i \odot \vw + \vb,
\end{align}
where $\odot$ denotes element-wise multiplication, and $\vw, \vb\in \R^d$ is the trainable parameter, initialized as $\vw = \vone, \vb = \vzero$. The text embeddings are given by $\vt_0, \vt_1 \in \R^d$. We omit the temperature parameter $\exp(t)$, as it can be absorbed into the text embeddings. 

We examine how the parameter $\vw = [w_1, \cdots, w_d]^\top$ changes after adaptation. Since scaling $\vw$ by a constant does not affect the direction of the image embedding $\vz_i$ or the prediction, we analyze the following normalized quantity, referred to as the \textit{feature importance}:
\begin{align}
    r_h = \frac{|w_h|}{\sum_{l=1}^d |w_l|}.
\end{align}
A larger $r_h$ indicates that the $h$-th feature contributes more to the prediction. Ideally, class-relevant features should have larger ratios, while domain-specific ones should have smaller ratios to ensure robustness to distribution shifts. Before adaptation, $\vw = \vone$, so all features have equal importance, i.e., $r_h = 1/d, \forall h$. The following Theorem \ref{thm:tta} analyzes how adaptation changes these feature importance.

\begin{restatable}{theorem}{tta}
\label{thm:tta}
    Perform one step of gradient descent on the support set $\sS$ to minimize the prediction entropy. Assume the learning rate $\eta$ is sufficiently small such that it does not change the sign of any element in $\vw$, we have
    \begin{align}
        r_h = \frac{1 + \eta \cdot (\ve_h \odot (\vt_1 - \vt_0))^\top \mM (\vt_1 - \vt_0))}{d + \eta \cdot (\vt_1 - \vt_0)^\top \mM (\vt_1 - \vt_0))}, 
    \end{align}
    where $\ve_h$ is the $h$-th standard basis (i.e., one-hot vector with 1 at the $h$-th position), and $\mM = \E_{j \in \sS}[ (p_{j0} p_{j1}) \cdot \vv_j \vv_j^\top ]$ denotes the probability-weighted uncentered covariance matrix (i.e., the second moment matrix) under the weighting $p_{j0} p_{j1}$. For better clarity, when $\mM$ is a diagonal matrix, i.e., features are uncorrelated, the equation above can be simplified as
    \begin{align}
        r_h = \frac{1 + \eta \cdot (t_{1h} - t_{0h})^2 \cdot M_{hh}}{d + \eta \cdot \sum_{l=1}^d (t_{1l} - t_{0l})^2 \cdot M_{ll}}, 
    \end{align}
    where $M_{hh} = \E_{j \in \sS} [(p_{j0} p_{j1}) \cdot v_{jh}^2]$ denotes the probability-weighted second moment of the $h$-th feature. 
\end{restatable}

\begin{remark}[Principle of TTA]
    Theorem \ref{thm:tta} implies that entropy minimization amplifies the contribution of features whose $(t_{1h} - t_{0h})^2 \cdot M_{hh}$ is above average, while suppressing those with smaller values. In other words, it performs \textit{feature reweighting} based on two factors: 
    \begin{itemize}
        \item Larger $(t_{1h} - t_{0h})^2$, which indicates that the \textbf{text embeddings} of the two classes differ more along dimension $h$, meaning that \textit{this feature is discriminative in the text description}; 
        \item Larger $M_{hh} = \E[(p_{i0} p_{i1}) \cdot v_{ih}^2]$, which captures greater variance within the \textbf{image embeddings} in target domain, reflecting \textit{distributional information of target domain}. 
    \end{itemize}
    To build an intuition for this reweighting effect, we qualitatively discuss how features of different types may behave under this adaptation: 
    \begin{itemize}
        \item \textbf{Class-relevant features.} These features vary significantly across samples from different classes. This leads to a large variance $M_{hh}$, which the model interprets as containing discriminative signals; such features are therefore amplified during adaptation. 
        \item \textbf{Domain-relevant features.} Under mixed-domain shift, the dataset contains samples from multiple domains. As a result, a domain-relevant feature can also exhibit large variance $M_{hh}$. Therefore, domain-specific features may \textit{not} be sufficiently suppressed and can even be mistakenly enhanced, leading to degraded performance.
    \end{itemize}
\end{remark}

\begin{remark}[Effect of active sample selection]
    The two criteria in active sample selection can enhance the effectiveness of entropy minimization under mixed-domain shifts:
    \begin{itemize}
        \item \textbf{Domain consistency.} Retrieving samples that are close to the query within each pseudo-class ensures that class-irrelevant features, including domain-relevant ones, are more consistent within the constructed support set. As a result, their variance $M_{hh}$ becomes smaller, allowing them to be more effectively suppressed.
        \item \textbf{Prediction balance.} Maintaining a balanced number of samples across pseudo-classes keeps the class-relevant features exhibiting large variance $M_{hh}$, so that these discriminative features are still amplified during adaptation.
    \end{itemize}
\end{remark}

\section{Experiments}

\renewcommand{\meansd}[2]{#1}  

\begin{table*}
    \centering
    \caption{Mean accuracy (\%) on corruption benchmarks under mixture of 15 corruptions. Best and second-best results are shown in \textbf{bold} and \underline{underlined}, respectively. Error bars are reported in Appendix \ref{appendix:exp:error_bars}.}
    \label{tab:corruption}
    \footnotesize	
    \setlength{\tabcolsep}{3.0pt}{
        \begin{tabular}{llccccccccccccccc >{\columncolor{cvprblue!15}}c}
            \toprule
            & & \multicolumn{16}{c}{ViT-B/32 on CIFAR-10-C} \\
            \cmidrule(lr){3-18}
            Method & Venue & \multicolumn{3}{c}{Noise} & \multicolumn{4}{c}{Blur} & \multicolumn{4}{c}{Weather} & \multicolumn{4}{c}{Digital} & \multirow{2.5}{*}{Avg.} \cellcolor{white} \\
            \cmidrule(lr){3-5} \cmidrule(lr){6-9} \cmidrule(lr){10-13} \cmidrule(lr){14-17}
            & & Gauss. & Shot & Impul. 
            & Defoc. & Glass & Motion & Zoom  
            & Snow & Frost & Fog & Brit. 
            & Contr. & Elastic & Pixel & JPEG \\
            \midrule
            CLIP \cite{clip} & ICML'21 &
                35.5 & 40.0 & 43.2 & 70.0 & 41.4 & 64.5 & 70.2 & 70.8 & 72.3 & 66.7 & 81.4 & 64.5 & 59.6 & 48.2 & 56.7 & 59.0 \\
            Ensemble & - &
                38.6 & 42.6 & 42.7 & 72.4 & 43.9 & 66.6 & 71.6 & 73.8 & 75.7 & 69.0 & 83.6 & 67.0 & 61.9 & 51.8 & 58.6 & 61.3 \\
            Tent \cite{tent} & ICLR'21 & 
                \meansd{39.5}{0.6} & \meansd{43.1}{0.4} & \meansd{46.4}{0.3} & \meansd{\underline{77.8}}{0.3} & \meansd{56.6}{0.3} & \meansd{72.7}{0.4} & \meansd{77.9}{0.4} & \meansd{\underline{79.7}}{0.2} & \meansd{79.9}{0.2} & \meansd{74.0}{0.3} & \meansd{86.9}{0.1} & \meansd{75.5}{0.1} & \meansd{69.7}{0.3} & \meansd{60.3}{0.1} & \meansd{63.9}{0.4} & \meansd{66.9}{0.2} \\
            NOTE \cite{note} & NeurIPS'22 & 
                \meansd{44.9}{4.4} & \meansd{48.3}{3.9} & \meansd{48.9}{1.7} & \meansd{77.1}{0.5} & \meansd{57.0}{0.7} & \meansd{74.0}{1.4} & \meansd{77.6}{0.7} & \meansd{78.0}{0.8} & \meansd{79.2}{0.9} & \meansd{73.0}{1.4} & \meansd{86.5}{0.5} & \meansd{73.5}{1.8} & \meansd{68.7}{1.1} & \meansd{57.9}{1.5} & \meansd{62.2}{0.7} & \meansd{67.1}{0.8} \\
            SAR \cite{sar} & ICLR'23 & 
                \meansd{48.8}{0.7} & \meansd{51.8}{0.6} & \meansd{51.0}{0.5} & \meansd{77.5}{0.2} & \meansd{60.5}{0.2} & \meansd{74.8}{0.5} & \meansd{\textbf{79.2}}{0.3} & \meansd{79.6}{0.2} & \meansd{\textbf{80.7}}{0.4} & \meansd{76.6}{0.4} & \meansd{\underline{87.4}}{0.1} & \meansd{77.0}{0.2} & \meansd{71.0}{0.3} & \meansd{57.5}{0.4} & \meansd{62.8}{0.3} & \meansd{69.1}{0.2} \\
            RoTTA \cite{rotta} & CVPR'23 & 
                \meansd{\underline{53.1}}{1.3} & \meansd{\underline{56.4}}{1.4} & \meansd{\underline{51.3}}{0.9} & \meansd{\textbf{78.1}}{0.5} & \meansd{\underline{61.0}}{1.0} & \meansd{75.2}{0.2} & \meansd{\underline{79.0}}{0.4} & \meansd{79.0}{0.6} & \meansd{80.1}{1.0} & \meansd{\textbf{80.0}}{0.4} & \meansd{85.6}{0.5} & \meansd{\textbf{82.9}}{0.4} & \meansd{\textbf{73.7}}{0.9} & \meansd{\underline{65.3}}{0.9} & \meansd{\textbf{69.9}}{0.7} & \meansd{\underline{71.4}}{0.4} \\
            TDA \cite{tda} & CVPR'24 & 
                \meansd{39.8}{0.4} & \meansd{42.8}{0.4} & \meansd{43.4}{0.3} & \meansd{73.5}{0.3} & \meansd{45.4}{0.2} & \meansd{67.6}{0.2} & \meansd{73.2}{0.2} & \meansd{74.3}{0.2} & \meansd{76.4}{0.3} & \meansd{69.8}{0.2} & \meansd{84.0}{0.2} & \meansd{66.6}{0.2} & \meansd{62.5}{0.1} & \meansd{52.1}{0.3} & \meansd{57.7}{0.4} & \meansd{61.9}{0.1} \\
            DMN-ZS \cite{dmn} & CVPR'24 & 
                \meansd{38.4}{0.1} & \meansd{41.7}{0.1} & \meansd{42.4}{0.1} & \meansd{73.1}{0.0} & \meansd{44.6}{0.0} & \meansd{67.2}{0.1} & \meansd{72.8}{0.1} & \meansd{74.6}{0.1} & \meansd{76.3}{0.0} & \meansd{69.5}{0.0} & \meansd{84.0}{0.0} & \meansd{66.5}{0.0} & \meansd{62.3}{0.0} & \meansd{52.4}{0.1} & \meansd{58.3}{0.1} & \meansd{61.6}{0.0} \\
            WATT-S \cite{watt} & NeurIPS'24 & 
                \meansd{47.4}{0.3} & \meansd{49.7}{0.1} & \meansd{49.5}{0.3} & \meansd{77.0}{0.2} & \meansd{54.3}{0.2} & \meansd{72.9}{0.2} & \meansd{77.3}{0.2} & \meansd{77.8}{0.2} & \meansd{79.5}{0.1} & \meansd{74.8}{0.2} & \meansd{86.4}{0.1} & \meansd{74.1}{0.2} & \meansd{67.8}{0.2} & \meansd{57.2}{0.4} & \meansd{62.9}{0.2} & \meansd{67.2}{0.0} \\
            CLIPArTT \cite{clipartt} & WACV'25 & 
                \meansd{37.2}{0.2} & \meansd{41.5}{0.2} & \meansd{44.8}{0.2} & \meansd{70.5}{0.1} & \meansd{48.9}{0.2} & \meansd{65.5}{0.2} & \meansd{70.7}{0.3} & \meansd{72.9}{0.3} & \meansd{75.3}{0.1} & \meansd{67.4}{0.2} & \meansd{82.8}{0.3} & \meansd{66.9}{0.1} & \meansd{61.9}{0.3} & \meansd{58.5}{0.2} & \meansd{59.5}{0.3} & \meansd{61.6}{0.1} \\
            Mint \cite{mint} & NeurIPS'25 & 
                \meansd{47.9}{0.1} & \meansd{50.9}{0.1} & \meansd{50.3}{0.1} & \meansd{74.1}{0.0} & \meansd{51.3}{0.1} & \meansd{\underline{75.8}}{0.0} & \meansd{76.9}{0.0} & \meansd{76.4}{0.1} & \meansd{77.0}{0.1} & \meansd{73.8}{0.1} & \meansd{85.8}{0.1} & \meansd{73.8}{0.0} & \meansd{65.8}{0.1} & \meansd{47.3}{0.0} & \meansd{56.4}{0.1} & \meansd{65.6}{0.0} \\
            {\algname} & - & 
                \meansd{\textbf{61.0}}{0.4} & \meansd{\textbf{63.3}}{0.3} & \meansd{\textbf{56.0}}{0.2} & \meansd{77.2}{0.1} & \meansd{\textbf{64.1}}{0.2} & \meansd{\textbf{77.2}}{0.1} & \meansd{78.9}{0.3} & \meansd{\textbf{80.9}}{0.1} & \meansd{\underline{80.3}}{0.2} & \meansd{\underline{78.0}}{0.2} & \meansd{\textbf{88.1}}{0.2} & \meansd{\underline{79.8}}{0.2} & \meansd{\underline{72.0}}{0.1} & \meansd{\textbf{65.5}}{0.1} & \meansd{\underline{67.5}}{0.2} & \meansd{\textbf{72.7}}{0.1} \\
            \midrule 
            & & \multicolumn{16}{c}{ViT-B/16 on CIFAR-100-C} \\
            \cmidrule(lr){3-18}
            Method & Venue & \multicolumn{3}{c}{Noise} & \multicolumn{4}{c}{Blur} & \multicolumn{4}{c}{Weather} & \multicolumn{4}{c}{Digital} & \multirow{2.5}{*}{Avg.} \cellcolor{white}\\
            \cmidrule(lr){3-5} \cmidrule(lr){6-9} \cmidrule(lr){10-13} \cmidrule(lr){14-17}
            & & Gauss. & Shot & Impul. 
            & Defoc. & Glass & Motion & Zoom  
            & Snow & Frost & Fog & Brit. 
            & Contr. & Elastic & Pixel & JPEG \\
            \midrule
            CLIP \cite{clip} & ICML'21 &
                19.7 & 21.4 & 25.3 & 42.5 & 20.2 & 43.1 & 48.0 & 48.4 & 49.7 & 41.7 & 57.0 & 34.5 & 29.2 & 23.9 & 32.4 & 35.8 \\
            Ensemble & - &
                22.9 & 24.4 & 29.6 & 43.6 & 20.1 & 43.7 & 48.7 & 48.9 & 50.4 & 41.8 & 58.1 & 35.3 & 29.2 & 26.3 & 33.6 & 37.1 \\
            Tent \cite{tent} & ICLR'21 & 
                \meansd{24.9}{0.2} & \meansd{26.7}{0.2} & \meansd{34.4}{0.1} & \meansd{49.7}{0.1} & \meansd{23.9}{0.2} & \meansd{47.5}{0.1} & \meansd{53.9}{0.1} & \meansd{52.9}{0.1} & \meansd{51.4}{0.2} & \meansd{45.3}{0.2} & \meansd{62.9}{0.2} & \meansd{43.6}{0.1} & \meansd{32.0}{0.1} & \meansd{\underline{31.7}}{0.1} & \meansd{37.1}{0.1} & \meansd{41.2}{0.0} \\
            NOTE \cite{note} & NeurIPS'22 & 
                \meansd{25.7}{1.4} & \meansd{27.3}{1.8} & \meansd{35.4}{1.2} & \meansd{49.9}{0.7} & \meansd{23.9}{1.2} & \meansd{47.6}{0.7} & \meansd{54.2}{1.1} & \meansd{52.7}{0.7} & \meansd{51.8}{1.2} & \meansd{45.9}{0.5} & \meansd{63.7}{0.9} & \meansd{44.0}{1.1} & \meansd{32.4}{1.4} & \meansd{30.2}{2.1} & \meansd{36.0}{0.8} & \meansd{41.4}{0.8} \\
            SAR \cite{sar} & ICLR'23 & 
                \meansd{\underline{28.4}}{0.2} & \meansd{\underline{30.5}}{0.3} & \meansd{\underline{38.5}}{0.2} & \meansd{\underline{50.4}}{0.2} & \meansd{24.6}{0.4} & \meansd{\underline{49.2}}{0.2} & \meansd{\underline{55.1}}{0.2} & \meansd{54.1}{0.2} & \meansd{53.4}{0.2} & \meansd{47.2}{0.2} & \meansd{\underline{64.0}}{0.3} & \meansd{45.0}{0.2} & \meansd{33.6}{0.1} & \meansd{29.7}{0.2} & \meansd{37.4}{0.2} & \meansd{\underline{42.7}}{0.1} \\
            RoTTA \cite{rotta} & CVPR'23 & 
                \meansd{22.9}{0.6} & \meansd{24.5}{0.8} & \meansd{32.4}{1.0} & \meansd{47.9}{1.2} & \meansd{\underline{25.4}}{1.2} & \meansd{46.5}{1.2} & \meansd{50.1}{0.9} & \meansd{50.5}{1.4} & \meansd{50.7}{1.1} & \meansd{\underline{50.8}}{0.7} & \meansd{58.3}{1.0} & \meansd{\textbf{53.1}}{0.7} & \meansd{\textbf{37.7}}{0.8} & \meansd{29.9}{1.0} & \meansd{36.2}{0.8} & \meansd{41.1}{0.8} \\
            TDA \cite{tda} & CVPR'24 & 
                \meansd{23.4}{0.1} & \meansd{25.2}{0.2} & \meansd{30.3}{0.2} & \meansd{44.1}{0.2} & \meansd{20.3}{0.2} & \meansd{43.8}{0.1} & \meansd{49.1}{0.2} & \meansd{49.3}{0.2} & \meansd{51.3}{0.1} & \meansd{42.2}{0.1} & \meansd{58.7}{0.1} & \meansd{36.1}{0.1} & \meansd{29.3}{0.1} & \meansd{26.4}{0.2} & \meansd{33.6}{0.1} & \meansd{37.5}{0.0} \\
            DMN-ZS \cite{dmn} & CVPR'24 & 
                \meansd{22.9}{0.1} & \meansd{24.4}{0.2} & \meansd{29.6}{0.1} & \meansd{44.3}{0.1} & \meansd{20.2}{0.1} & \meansd{44.1}{0.0} & \meansd{49.4}{0.1} & \meansd{49.7}{0.1} & \meansd{51.1}{0.1} & \meansd{42.2}{0.1} & \meansd{58.8}{0.1} & \meansd{35.5}{0.1} & \meansd{29.4}{0.0} & \meansd{26.2}{0.1} & \meansd{34.0}{0.1} & \meansd{37.5}{0.0} \\
            WATT-S \cite{watt} & NeurIPS'24 & 
                \meansd{26.4}{0.2} & \meansd{28.4}{0.3} & \meansd{35.6}{0.3} & \meansd{50.2}{0.3} & \meansd{24.2}{0.3} & \meansd{48.5}{0.2} & \meansd{54.4}{0.2} & \meansd{\underline{54.2}}{0.3} & \meansd{\underline{54.0}}{0.2} & \meansd{47.7}{0.3} & \meansd{63.5}{0.4} & \meansd{43.3}{0.2} & \meansd{34.0}{0.2} & \meansd{31.4}{0.2} & \meansd{\underline{37.7}}{0.4} & \meansd{42.2}{0.1} \\
            CLIPArTT \cite{clipartt} & WACV'25 & 
                \meansd{21.1}{0.3} & \meansd{22.8}{0.2} & \meansd{29.4}{0.2} & \meansd{46.3}{0.1} & \meansd{24.5}{0.3} & \meansd{45.0}{0.1} & \meansd{51.1}{0.1} & \meansd{51.5}{0.1} & \meansd{51.8}{0.2} & \meansd{44.0}{0.3} & \meansd{61.1}{0.2} & \meansd{39.0}{0.2} & \meansd{33.2}{0.2} & \meansd{28.3}{0.3} & \meansd{36.7}{0.1} & \meansd{39.0}{0.1} \\
            Mint \cite{mint} & NeurIPS'25 & 
                \meansd{22.3}{0.0} & \meansd{24.5}{0.1} & \meansd{33.0}{0.1} & \meansd{49.2}{0.1} & \meansd{22.4}{0.1} & \meansd{47.8}{0.1} & \meansd{54.2}{0.1} & \meansd{52.4}{0.0} & \meansd{51.6}{0.1} & \meansd{47.7}{0.1} & \meansd{63.7}{0.1} & \meansd{42.2}{0.1} & \meansd{31.8}{0.0} & \meansd{24.0}{0.1} & \meansd{34.3}{0.0} & \meansd{40.1}{0.0} \\
            {\algname} & - & 
                \meansd{\textbf{30.3}}{0.3} & \meansd{\textbf{32.9}}{0.4} & \meansd{\textbf{42.8}}{0.3} & \meansd{\textbf{52.5}}{0.3} & \meansd{\textbf{29.1}}{0.4} & \meansd{\textbf{52.2}}{0.2} & \meansd{\textbf{55.5}}{0.1} & \meansd{\textbf{55.4}}{0.1} & \meansd{\textbf{55.1}}{0.2} & \meansd{\textbf{52.6}}{0.2} & \meansd{\textbf{65.9}}{0.2} & \meansd{\underline{51.0}}{0.2} & \meansd{\underline{36.8}}{0.2} & \meansd{\textbf{40.4}}{0.4} & \meansd{\textbf{39.6}}{0.2} & \meansd{\textbf{46.1}}{0.1} \\
            \midrule 
            & & \multicolumn{16}{c}{ViT-L/14 on ImageNet-C} \\
            \cmidrule(lr){3-18}
            Method & Venue & \multicolumn{3}{c}{Noise} & \multicolumn{4}{c}{Blur} & \multicolumn{4}{c}{Weather} & \multicolumn{4}{c}{Digital} & \multirow{2.5}{*}{Avg.} \cellcolor{white}\\
            \cmidrule(lr){3-5} \cmidrule(lr){6-9} \cmidrule(lr){10-13} \cmidrule(lr){14-17}
            & & Gauss. & Shot & Impul. 
            & Defoc. & Glass & Motion & Zoom  
            & Snow & Frost & Fog & Brit. 
            & Contr. & Elastic & Pixel & JPEG \\
            \midrule
            CLIP \cite{clip} & ICML'21 &
                27.4 & 29.4 & 28.7 & 34.6 & 25.3 & 41.0 & 36.7 & 49.8 & 44.1 & 49.7 & 65.4 & 35.1 & 30.3 & 53.5 & 42.2 & 39.6 \\
            Ensemble & - &
                29.1 & 30.4 & 30.1 & 37.5 & 27.3 & 44.2 & 39.2 & 52.4 & 46.4 & 52.6 & 67.8 & 34.4 & 32.4 & 56.2 & 44.2 & 41.6 \\
            Tent \cite{tent} & ICLR'21 & 
                \meansd{36.3}{0.4} & \meansd{38.0}{0.1} & \meansd{38.5}{0.1} & \meansd{\underline{40.3}}{0.3} & \meansd{37.4}{0.5} & \meansd{\underline{47.7}}{0.3} & \meansd{\underline{43.5}}{0.4} & \meansd{54.2}{0.4} & \meansd{49.5}{0.2} & \meansd{56.4}{0.3} & \meansd{67.8}{0.2} & \meansd{45.4}{0.7} & \meansd{40.6}{0.5} & \meansd{57.6}{0.3} & \meansd{46.6}{0.3} & \meansd{\underline{46.7}}{0.1} \\
            NOTE \cite{note} & NeurIPS'22 & 
                \meansd{36.3}{0.4} & \meansd{38.0}{0.2} & \meansd{38.6}{0.3} & \meansd{40.0}{0.4} & \meansd{37.3}{0.4} & \meansd{47.6}{0.3} & \meansd{\underline{43.5}}{0.3} & \meansd{\underline{54.4}}{0.4} & \meansd{49.7}{0.5} & \meansd{56.4}{0.5} & \meansd{67.8}{0.3} & \meansd{44.5}{0.4} & \meansd{40.8}{0.4} & \meansd{57.5}{0.3} & \meansd{46.2}{0.5} & \meansd{46.6}{0.1} \\
            SAR \cite{sar} & ICLR'23 & 
                \meansd{36.3}{0.8} & \meansd{38.3}{0.4} & \meansd{38.5}{0.4} & \meansd{39.2}{0.7} & \meansd{\underline{38.5}}{0.8} & \meansd{46.8}{0.5} & \meansd{43.4}{0.6} & \meansd{53.9}{0.3} & \meansd{\underline{50.4}}{0.5} & \meansd{\underline{56.6}}{0.4} & \meansd{66.8}{0.2} & \meansd{\underline{45.5}}{0.1} & \meansd{\underline{41.8}}{0.8} & \meansd{56.2}{0.6} & \meansd{46.6}{0.5} & \meansd{46.6}{0.3} \\
            RoTTA \cite{rotta} & CVPR'23 & 
                \meansd{\underline{37.4}}{0.5} & \meansd{\underline{38.7}}{0.4} & \meansd{\underline{39.6}}{0.7} & \meansd{36.9}{0.5} & \meansd{34.5}{0.5} & \meansd{44.0}{0.5} & \meansd{39.3}{0.7} & \meansd{53.5}{0.6} & \meansd{48.6}{0.6} & \meansd{54.2}{0.7} & \meansd{67.2}{0.3} & \meansd{43.0}{0.8} & \meansd{39.7}{0.6} & \meansd{56.3}{0.5} & \meansd{\underline{48.4}}{0.6} & \meansd{45.4}{0.4} \\
            TDA \cite{tda} & CVPR'24 & 
                \meansd{29.6}{0.1} & \meansd{31.0}{0.1} & \meansd{31.5}{0.2} & \meansd{38.1}{0.1} & \meansd{28.9}{0.1} & \meansd{44.6}{0.1} & \meansd{40.0}{0.1} & \meansd{53.4}{0.1} & \meansd{47.5}{0.1} & \meansd{53.3}{0.2} & \meansd{\underline{68.5}}{0.1} & \meansd{39.1}{0.2} & \meansd{33.4}{0.2} & \meansd{57.1}{0.2} & \meansd{45.1}{0.2} & \meansd{42.7}{0.0} \\
            DMN-ZS \cite{dmn} & CVPR'24 & 
                \meansd{29.2}{0.1} & \meansd{30.5}{0.1} & \meansd{30.2}{0.0} & \meansd{37.6}{0.1} & \meansd{27.6}{0.1} & \meansd{44.4}{0.1} & \meansd{39.2}{0.1} & \meansd{52.5}{0.0} & \meansd{46.6}{0.1} & \meansd{52.7}{0.0} & \meansd{67.9}{0.0} & \meansd{33.8}{0.1} & \meansd{32.5}{0.0} & \meansd{56.5}{0.1} & \meansd{44.5}{0.1} & \meansd{41.7}{0.0} \\
            WATT-S \cite{watt} & NeurIPS'24 & 
                \meansd{31.8}{0.1} & \meansd{33.1}{0.1} & \meansd{33.6}{0.2} & \meansd{39.2}{0.2} & \meansd{30.7}{0.3} & \meansd{45.7}{0.1} & \meansd{41.3}{0.1} & \meansd{53.5}{0.2} & \meansd{47.3}{0.3} & \meansd{53.7}{0.2} & \meansd{67.9}{0.3} & \meansd{40.1}{0.1} & \meansd{34.7}{0.3} & \meansd{57.3}{0.1} & \meansd{45.6}{0.3} & \meansd{43.7}{0.0} \\
            CLIPArTT \cite{clipartt} & WACV'25 &
                \meansd{28.2}{0.3} & \meansd{30.2}{0.3} & \meansd{29.8}{0.2} & \meansd{35.0}{0.2} & \meansd{26.5}{0.3} & \meansd{41.4}{0.2} & \meansd{37.3}{0.1} & \meansd{50.3}{0.3} & \meansd{43.7}{0.3} & \meansd{49.7}{0.2} & \meansd{65.0}{0.2} & \meansd{38.5}{0.2} & \meansd{31.0}{0.3} & \meansd{53.5}{0.2} & \meansd{42.3}{0.3} & \meansd{40.2}{0.1} \\
            Mint \cite{mint} & NeurIPS'25 & 
                \meansd{33.8}{0.2} & \meansd{35.2}{0.2} & \meansd{35.8}{0.2} & \meansd{39.3}{0.2} & \meansd{33.3}{0.3} & \meansd{46.2}{0.2} & \meansd{41.5}{0.2} & \meansd{53.6}{0.3} & \meansd{47.4}{0.2} & \meansd{53.3}{0.1} & \meansd{67.7}{0.1} & \meansd{36.3}{0.5} & \meansd{38.2}{0.3} & \meansd{\underline{58.1}}{0.1} & \meansd{47.4}{0.7} & \meansd{44.5}{0.1} \\
            {\algname} & - & 
                \meansd{\textbf{37.5}}{0.2} & \meansd{\textbf{38.8}}{0.1} & \meansd{\textbf{41.1}}{0.2} & \meansd{\textbf{42.2}}{0.3} & \meansd{\textbf{39.6}}{0.4} & \meansd{\textbf{49.9}}{0.4} & \meansd{\textbf{46.2}}{0.4} & \meansd{\textbf{57.3}}{0.3} & \meansd{\textbf{51.6}}{0.1} & \meansd{\textbf{59.4}}{0.1} & \meansd{\textbf{68.7}}{0.2} & \meansd{\textbf{46.4}}{0.4} & \meansd{\textbf{45.2}}{0.2} & \meansd{\textbf{59.4}}{0.2} & \meansd{\textbf{54.7}}{0.2} & \meansd{\textbf{49.2}}{0.1} \\
            \bottomrule 
        \end{tabular}
    }
\end{table*}
\begin{table}
    \centering
    \caption{Mean accuracy (\%) on DomainNet under mixture of six domains. Best and second-best results are shown in \textbf{bold} and \underline{underlined}, respectively. Error bars are reported in Appendix \ref{appendix:exp:error_bars}.}
    \label{tab:domainnet}
    \footnotesize	
    \setlength{\tabcolsep}{1.9pt}{
        \begin{tabular}{llcccccc >{\columncolor{cvprblue!15}}c}
            \toprule
            \multirow{2.5}{*}{Method} & \multirow{2.5}{*}{Venue} & \multicolumn{7}{c}{ViT-B/32 on DomainNet} \\
            \cmidrule(lr){3-9}
            & & clip & info & paint & quick & real & sketch & Avg. \cellcolor{white} \\
            \midrule
            CLIP \cite{clip} & ICML'21 & 
                67.6 & 41.5 & 62.7 & 12.8 & 81.0 & 58.1 & 54.0 \\
            Ensemble & - & 
                68.9 & 44.2 & 64.7 & 13.3 & 82.3 & 60.3 & 55.6 \\
            Tent \cite{tent} & ICLR'21 &
                \meansd{69.1}{0.0} & \meansd{44.3}{0.0} & \meansd{64.9}{0.0} & \meansd{13.0}{0.0} & \meansd{82.3}{0.0} & \meansd{60.3}{0.0} & \meansd{55.7}{0.0} \\
            NOTE \cite{note} & NeurIPS'22 &
                \meansd{\underline{69.6}}{0.1} & \meansd{\underline{45.1}}{0.0} & \meansd{65.0}{0.0} & \meansd{16.4}{0.1} & \meansd{82.1}{0.0} & \meansd{\underline{60.8}}{0.0} & \meansd{\underline{56.5}}{0.0} \\
            SAR \cite{sar} & ICLR'23 & 
                \meansd{69.4}{0.0} & \meansd{44.7}{0.0} & \meansd{65.2}{0.1} & \meansd{14.5}{0.0} & \meansd{\underline{82.5}}{0.0} & \meansd{\underline{60.8}}{0.0} & \meansd{56.2}{0.0} \\
            RoTTA \cite{rotta} & CVPR'23 & 
                \meansd{69.4}{0.1} & \meansd{44.8}{0.0} & \meansd{\underline{65.3}}{0.0} & \meansd{14.3}{0.0} & \meansd{82.4}{0.0} & \meansd{\underline{60.8}}{0.0} & \meansd{56.2}{0.0} \\
            TDA \cite{tda} & CVPR'24 & 
                \meansd{68.9}{0.0} & \meansd{44.7}{0.0} & \meansd{65.1}{0.1} & \meansd{14.6}{0.1} & \meansd{82.3}{0.0} & \meansd{60.5}{0.0} & \meansd{56.0}{0.0} \\
            DMN-ZS \cite{dmn} & CVPR'24 & 
                \meansd{68.9}{0.0} & \meansd{44.6}{0.0} & \meansd{64.8}{0.0} & \meansd{12.7}{0.0} & \meansd{\textbf{82.6}}{0.0} & \meansd{60.2}{0.0} & \meansd{55.6}{0.0} \\
            WATT-S \cite{watt} & NeurIPS'24 & 
                \meansd{68.3}{0.0} & \meansd{42.8}{0.1} & \meansd{63.5}{0.1} & \meansd{\underline{17.0}}{0.0} & \meansd{81.4}{0.0} & \meansd{59.9}{0.1} & \meansd{55.5}{0.0} \\
            CLIPArTT \cite{clipartt} & WACV'25 & 
                \meansd{67.2}{0.0} & \meansd{41.9}{0.0} & \meansd{62.9}{0.0} & \meansd{12.6}{0.0} & \meansd{80.8}{0.0} & \meansd{57.8}{0.1} & \meansd{53.9}{0.0} \\
            Mint \cite{mint} & NeurIPS'25 & 
                \meansd{69.3}{0.0} & \meansd{44.4}{0.0} & \meansd{65.0}{0.0} & \meansd{15.0}{0.0} & \meansd{82.3}{0.0} & \meansd{60.5}{0.0} & \meansd{56.1}{0.0} \\
            {\algname} & - & 
                \meansd{\textbf{70.1}}{0.0} & \meansd{\textbf{45.2}}{0.0} & \meansd{\textbf{65.8}}{0.0} & \meansd{\textbf{17.6}}{0.1} & \meansd{82.4}{0.0} & \meansd{\textbf{61.5}}{0.0} & \meansd{\textbf{57.1}}{0.0} \\
            \bottomrule 
        \end{tabular}
    }
\end{table}

In this section, we conduct experiments to answer the following research questions:
\begin{itemize}
    \item \textbf{RQ1}: Compared with existing TTA methods, is {\algname} more robust under mixture-of-domain settings and able to achieve better performance?
    \item \textbf{RQ2}: What types of samples does {\algname} retrieve as the support set for adaptation?
    \item \textbf{RQ3}: Does the design of {\algname} ensure computational efficiency in practice?
\end{itemize}


\paragraph{Datasets}
Following the evaluation protocol of SAR \cite{sar}, we conduct experiments under mixed distribution shifts on standard corruption benchmarks \cite{corruption}: ViT-B/32 \cite{vit} on CIFAR-10-C \cite{cifar}, ViT-B/16 on CIFAR-100-C, and ViT-L/14 on ImageNet-C \cite{imagenet}. We further evaluate ViT-B/32 on DomainNet \cite{domainnet} to assess the performance of {\algname} under mixture of domain shifts. Note that prior VLM studies \cite{tpt,tda,dmn} typically evaluate models under natural distribution shifts (e.g., ImageNet-V2, -A, -R, -Sketch)~\cite{recht2019imagenet, hendrycks2021natural, hendrycks2021many, wang2019learning} or cross-domain benchmarks (e.g., Flower102, Food101)~\cite{bossard14}. However, these datasets each represent a single domain and are therefore not suitable for evaluating mixed-domain shifts. 

\paragraph{Baselines}
We compare our method with a wide range of baselines. For generic TTA methods, we include Tent \cite{tent}, NOTE \cite{note}, SAR \cite{sar}, and RoTTA \cite{rotta}. Among them, SAR is specifically designed for wild test scenarios, including mixed-domain settings, while NOTE and RoTTA also employ prediction-balanced memory banks but do not explicitly account for domain mixtures.
For VLM-based TTA methods, we consider two memory-based approaches, TDA \cite{tda} and DMN-ZS \cite{dmn}, as well as three recent and highly competitive baselines: WATT-S \cite{watt}, CLIPArTT \cite{clipartt}, and Mint \cite{mint}. CLIP and Ensemble represent the zero-shot performance of CLIP using a single template (``\texttt{a photo of a \{class\}}'') and seven templates from \cite{tip-adaptor}, respectively. For all TTA methods, except from CLIPArTT which modifies the prompts, also use the seven templates. More details of experiments setup, including hyperparameter, are provided in Appendix \ref{appendix:exp:setup}. 


\paragraph{Main results (RQ1)}
We evaluate all methods under mixed-domain settings, where test samples from multiple domains are fully interleaved during adaptation rather than evaluated separately. Following prior works \cite{watt,mint}, we report the accuracy for each domain individually and then average them across all domains for a fair comparison. The results are summarized in Table \ref{tab:corruption} and Table \ref{tab:domainnet}. Overall, {\algname} consistently achieves the best performance across all datasets and architectures. Specifically, it improves the average accuracy by +1.3\% on CIFAR-10-C, +3.4\% on CIFAR-100-C, +2.5\% on ImageNet-C, and +0.6\% on DomainNet compared with the strongest baseline. 


\begin{figure}
    \centering
    \includegraphics[width=1.0\linewidth]{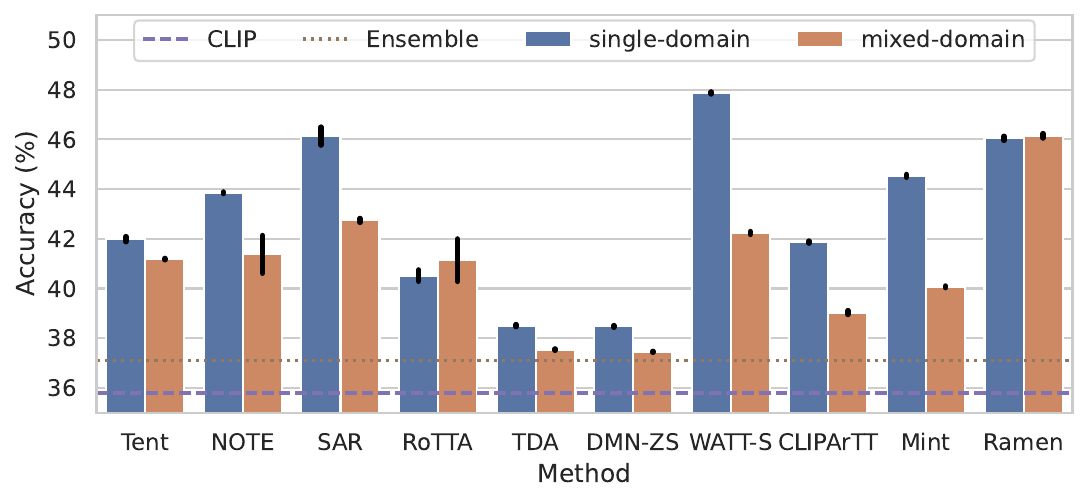}
    \caption{Performance comparison of TTA methods under single-domain and mixed-domain shifts on CIFAR-100-C. Results on other datasets are provided in Appendix \ref{appendix:exp:single_vs_mixed}. }
    \label{fig:single_vs_mixed}
\end{figure}

\paragraph{Robustness to domain mixture (RQ1)}
We further compare the performance of different methods under both single-domain and mixed-domain settings.
In the single-domain evaluation, the model is tested on each domain separately without mixing test samples, and we report the average accuracy across all domains. 
The results are presented in Figure~\ref{fig:single_vs_mixed}.
We observe that existing baseline methods, especially recent VLM-based TTA approaches, perform well under the single-domain setting.
However, most of them suffer from significant performance drops when domains are mixed, while earlier methods are also affected to a lesser extent.
In contrast, {\algname} maintains consistently strong performance across both settings.


\begin{figure}
    \centering
    \includegraphics[width=0.9\linewidth]{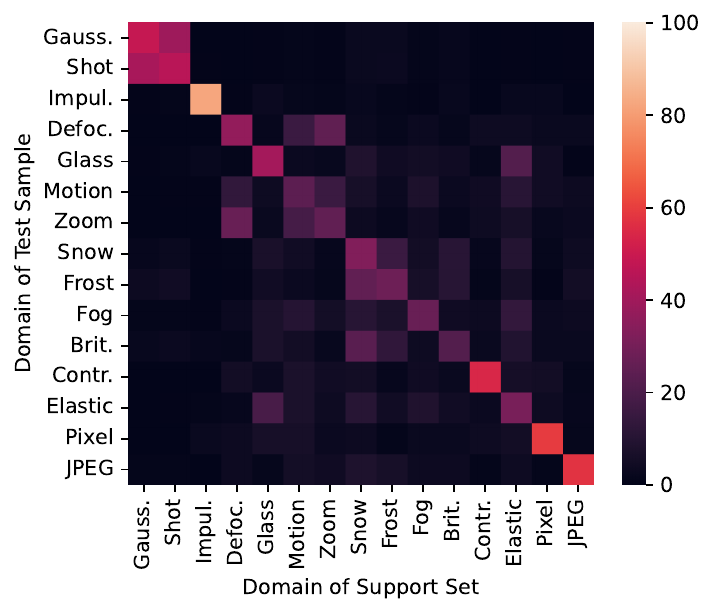}
    \caption{Visualization of active sample selection on CIFAR-100-C. Each entry $(i, j)$ indicates the average proportion (\%) of support samples from domain $j$ when the test sample comes from domain $i$. On average, 40.9\% of the support samples come from the same domain as the test sample, which is substantially higher than random selection (6.7\%). Results on other datasets are provided in Appendix \ref{appendix:exp:vis_mat}.}
    \label{fig:visual}
\end{figure}

\paragraph{Visualization of active sample selection (RQ2)}
To further understand the mechanism of {\algname}'s active sample selection, we analyze the domain composition of its retrieved samples. In Figure~\ref{fig:visual}, each cell in row~$i$ and column~$j$ represents the proportion of samples from domain~$j$ within the customized support set when the new test sample (query) comes from domain~$i$. Hence, the diagonal elements indicate the frequency of retrieving samples from the same domain. We observe that on average, 40.9\% of the retrieved support samples come from the same domain as the query test sample, which is substantially higher than random selection (6.7\%). This indicates that {\algname}'s active sample selection effectively captures domain consistency, allowing the model to preferentially retrieve samples from similar domains and thereby perform more targeted adaptation.


\begin{figure*}[t]
    \centering
    \includegraphics[width=1.0\linewidth]{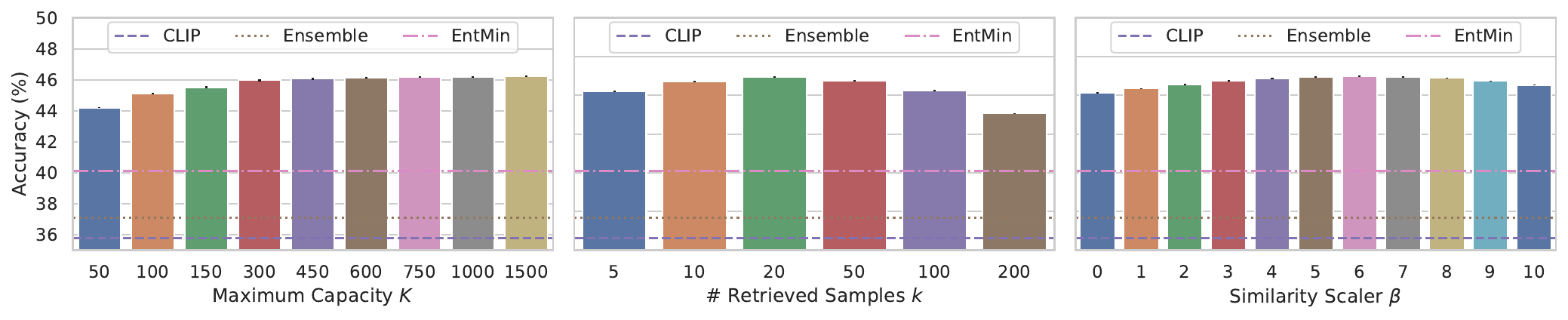}
    \caption{Hyperparameter sensitivity on CIFAR-100-C. (EntMin refers to entropy minimization without {\algname}.) }
    \label{fig:hp}
\end{figure*}

\begin{table}
    \centering
    \caption{Comparison of testing time on CIFAR-100-C.}
    \label{tab:time}
    \footnotesize
        \begin{tabular}{lrcc}
            \toprule
            Method & Testing Time & Acc. (\%) & Gain (\%) \\
            \midrule
            CLIP        & 5m27s     & 35.8  & - \\
            Tent        & 9m27s     & 41.2  & +5.4 \\
            NOTE        & 11m42s    & 41.4  & +5.6 \\
            SAR         & 15m50s    & 42.7  & +6.9 \\
            RoTTA       & 19m29s    & 41.1  & +5.3 \\
            TDA         & 7m59s     & 37.5  & +1.7 \\
            DMN-ZS      & 7m36s     & 37.5  & +1.7 \\
            WATT-S      & 18h54m50s & 42.2  & +6.4 \\
            CLIPArTT    & 3h05m36s  & 39.0  & +3.2 \\
            Mint        & 13m36s    & 40.1  & +4.3 \\
            \midrule
            {\algname}  & 14m08s    & 46.1  & +10.3 \\
            (w/o embed-grad cache)  & 115h42m18s & 46.1 & +10.3 \\
            \bottomrule 
        \end{tabular}
\end{table}

\paragraph{Efficiency (RQ3)}
We measure the testing time of {\algname} and baseline methods on the CIFAR-100-C dataset (150,000 images) to evaluate their efficiency. As shown in Table~\ref{tab:time}, the embedding-gradient cache in {\algname} substantially reduces the computational cost, achieving a 490$\times$ speedup compared to the naive implementation. As a result, the overall computation time of {\algname} is comparable to existing model-updating TTA methods, while delivering superior performance. We further analyze the \textit{GPU memory usage} of {\algname} in Appendix \ref{appendix:exp:memory}. 


\paragraph{Hyperparameter sensitivity}
In Figure~\ref{fig:hp}, we examine the sensitivity of {\algname} to its hyperparameters.
The maximum capacity $K$ controls the range of active sample selection: a larger $K$ provides a richer candidate pool and generally improves performance, though with diminishing returns once $K$ becomes sufficiently large.
The retrieval size $k$ determines how many samples are selected for adaptation.
As $k$ increases, the proportion of samples from different domains in the support set tends to grow, which weakens domain consistency; when $k$ is too small, the gradients are computed from too few samples and thus can be noisy. 
The similarity scaler $\beta$ plays a similar role to $k$. However, across a wide range of hyperparameter settings, {\algname} consistently improves TTA performance compared to vanilla entropy minimization (MinEnt), demonstrating its robustness to hyperparameter choices. 
Beyond the hyperparameters introduced by {\algname}, we also evaluate {\algname} under different learning rates and batch sizes, and observe consistent performance gains across all configurations. Appendix \ref{appendix:exp:hparams} provides full results of hyperpararameter sensitivity. 


\begin{figure}[t]
    \centering
    \includegraphics[width=0.85\linewidth]{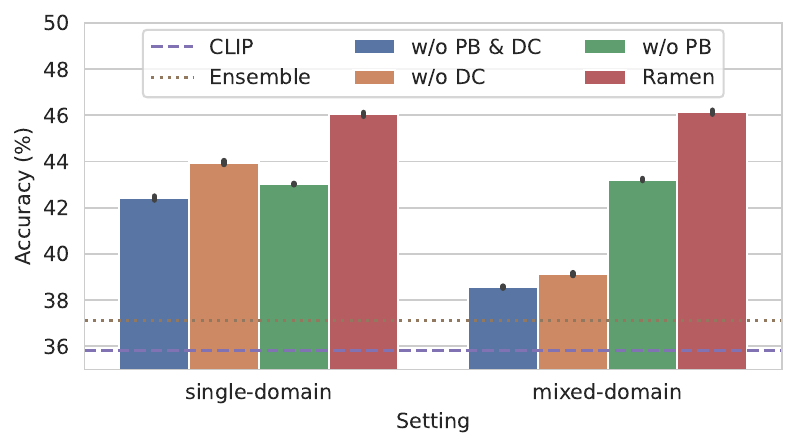}
    \caption{Ablation study. \textit{DC}: domain consistency, \textit{PB}: prediction balance. Both rules are beneficial to the performance of {\algname}.}
    \label{fig:ablation}
\end{figure}

\paragraph{Ablation study}
To further understand the effects of the two criteria in active sample selection, we conduct an ablation study with the following variants:
\begin{itemize}
    \item \textit{Without prediction balance (w/o PB)}: Instead of maintaining $C$ FIFO queues (one for each class) with capacity $K$ and retrieving the top-$k$ samples from each, we maintain a single FIFO queue without distinguishing classes, with capacity $C \cdot K$, and retrieve the top $C \cdot k$ samples. 
    \item \textit{Without domain consistency (w/o DC)}: We no longer select the top-$k$ most similar samples. Instead, $k$ samples are randomly chosen from the queue, and we set $\beta = 0$.
    \item \textit{Without both prediction balancing and domain consistency (w/o PB \& DC)}: This variant approximately corresponds to randomly selecting a subset of previously seen samples for model updates.
\end{itemize}
As shown in Figure~\ref{fig:ablation}, both the domain consistency and prediction balance criteria contribute to the overall performance of {\algname}.
Specifically, domain consistency makes {\algname} perform more consistently across both single-domain and mixed-domain settings, while prediction balance provides steady performance gains in both settings.
In Appendix~\ref{appendix:exp:ablation}, we further conduct an ablation study on the two components of the gradient aggregation weights: entropy and similarity. Results show that both are beneficial.

\section{Conclusion}
We present {\algname}, a framework for robust test-time adaptation under mixed-domain shifts. By actively selecting domain-consistent and prediction-balanced samples, it achieves stable and unbiased adaptation. With an efficient embedding-gradient cache, it enables fast updates without extra computation. Theoretical and empirical results demonstrate its robustness and efficiency.

\section*{Acknowledgement}
This work is supported by National Science Foundation under Award No. IIS-2416070, IIS-2117902. The views and conclusions are those of the authors and should not be interpreted as representing the official policies of the funding agencies or the government.

{
    \small
    \bibliographystyle{ieeenat_fullname}
    \bibliography{main}
}

\clearpage
\appendix
\onecolumn
{
\centering
\Large
\textbf{\thetitle}\\
\vspace{0.5em}Supplementary Material \\
\vspace{1.0em}
}
\addtocontents{toc}{\protect\setcounter{tocdepth}{2}}
\tableofcontents

\newpage
\section{Discussion and Details}
\label{appendix:discussion}

\subsection{More Related Works}
\label{appendix:discussion:related_work}

\paragraph{Differet variants of TTA settings}
\begin{itemize}
    \item \textit{Single-domain TTA} \cite{tent,matcha}: the most common TTA setting, where a pre-trained model is adapted to a single, consistent domain with a stable distribution. 
    \item \textit{Mixed-domain TTA} \cite{sar,sar2}: the main setting studied in this paper, where test samples from multiple domains are fully mixed, and the domain identity of each sample is unknown during testing. 
    \item \textit{Multi-target TTA} \cite{latte,atp}: assumes that a pre-trained model is adapted to a series of related but distinct domains, where domain relationships can be leveraged to improve performance. Unlike mixed-domain TTA, the domain label of each sample is known. 
    \item \textit{Continual TTA} \cite{cotta,rotta}: assumes that the model is adapted to a sequence of domains, simulating gradual distribution shifts over time. Although domain IDs are unavailable, domains exhibit temporal continuity—consecutive samples typically come from the same domain. 
    \item \textit{Noisy TTA} \cite{sotta}: assumes adaptation to a single domain, but the test data may contain noisy or out-of-distribution samples that are not drawn from the target distribution. This setting resembles mixed-domain TTA, but the noisy samples may not belong to any of the task classes, and their accuracy is not of interest.
\end{itemize}

\paragraph{Comparison to existing memory-based methods}
Existing memory-based approaches can be broadly categorized into two types: methods such as TDA \cite{tda} and DMN-ZS \cite{dmn} that directly store embeddings of previously seen test samples, and methods such as StatA \cite{stata} and BCA \cite{bca} that maintain statistics computed from these embeddings. {\algname} differs fundamentally from both categories in the following aspects.
\begin{itemize}
    \item \textit{What is cached.} Prior methods either store forward-pass embeddings or update summary statistics based on these embeddings. However, such designs do not support gradient-based model updates, and therefore cannot effectively address the mismatch between the visual encoder and the target domain, limiting their adaptation potential. In contrast, {\algname} caches both embeddings and gradients, enabling model updates and leading to greater adaptation gains.
    
    \item \textit{How the cache is used.} Existing methods typically apply a shared model to all test samples. For example, TDA uses the same cached embeddings regardless of the current input. Notably, while the cache itself may evolve over time, it is not conditioned on the specific test sample. In contrast, {\algname} performs active, per-sample retrieval to construct a support set tailored to each input, resulting in more robust and effective adaptation.
\end{itemize}


\subsection{Empirical Validation of Domain Consistency}
\label{appendix:discussion:empirical_validation}

In Subsection \ref{subsec:method:select}, we claimed that \textit{the more similar two image embeddings are, the more likely they originate from the same domain}. Here, we provide empirical validation. 

We conduct a study on CIFAR-100-C by computing pairwise embedding similarities, grouping sample pairs into 10 bins from high to low similarity (top 10\%, 10\%--20\%, $\ldots$, bottom 10\%), and measuring the fraction of same-domain pairs in each bin. Results are summarized in Table \ref{tab:empirical_validation}. 

We observe that higher similarity consistently corresponds to a higher same-domain ratio, supporting our claim.

\begin{table}[h!]
    \centering
    \caption{Empirical validation of embedding similarity for domain consistency. }
    \label{tab:empirical_validation}
    \footnotesize
    \begin{tabular}{lcccccccccc}
        \toprule
        Similarity bin (high $\rightarrow$ low) 
        & 1 & 2 & 3 & 4 & 5 & 6 & 7 & 8 & 9 & 10 \\
        \midrule
        Same-domain ratio (\%) 
        & 40.1 & 16.9 & 12.3 & 10.0 & 8.6 & 7.6 & 6.8 & 6.2 & 5.6 & 4.5 \\
        \bottomrule
    \end{tabular}
\end{table}


\newpage
\subsection{Reparameterization Trick}
\label{appendix:discussion:reparam}

In the main text, we describe our algorithm based on a single test sample, where computing the sample-level gradient is straightforward. However, when the batch size $B > 1$, performing forward and backward propagation for each sample individually has low parallelism and is thus inefficient. To address this, we adopt a reparameterization trick that enables more efficient computation of sample-level gradients.

Take LayerNorm \cite{layernorm} as an example, where the input has a shape of $L \times B \times D$, with $L$ being the sequence length, $B$ the batch size, and $D$ the hidden dimension. The affine parameters (weight and bias) normally have a dimension of $D$. In standard backpropagation, the gradient is averaged over the $B$ samples. To obtain per-sample gradients, we replicate the affine parameters $B$ times, resulting in a parameter dimension of $B \times D$, and change the reduction function in the entropy loss from mean to sum. During backpropagation, the resulting gradient also has a dimension of $B \times D$, corresponding directly to the gradient of each individual sample, since each sample’s forward pass is computed independently with its own replicated affine parameters. 

This approach is significantly more efficient than computing gradients one by one, because it leverages parallel computation in existing deep learning frameworks. Specifically, all samples share the same forward and backward computation graph, so their gradients can be computed in a single backward pass without serializing $B$ independent backward calls.

In addition, when the batch size is not excessively large, this reparameterization does not lead to a noticeable increase in memory usage. The majority of the computation graph remains unchanged, since only the LayerNorm affine parameters are duplicated, and these parameters account for a negligible fraction of the total model parameters (e.g., $<0.05\%$ on ViT-B/16). 
\newpage
\section{Missing Proofs}

\paragraph{Setup and assumptions}
Most TTA methods \cite{tent,note,sar,watt,clipartt,mint,panda,subtta} adapt models by updating the affine parameters in normalization layers. Therefore, we focus our analysis on this component. Common normalization layers \cite{batchnorm,groupnorm,layernorm,rmsnorm} share a same structure: a normalization step followed by an element-wise affine/linear transformation. Although the normalization step differs across layer types, the parameterization of the affine/linear transformation remains similar. Following \cite{mint}, we consider a single normalization layer in a binary classification setting. Let $\vv_i \in \R^d$ denote the intermediate normalized feature, obtained after the normalization step but before applying its affine transformation. The final image embedding is given by an element-wise linear transformation: 
\begin{align}
    \vz_i = \vv_i \odot \vw + \vb,
\end{align}
where $\odot$ denotes element-wise multiplication, and $\vw, \vb\in \R^d$ is the trainable parameter, initialized as $\vw = \vone, \vb = \vzero$. The text embeddings are given by $\vt_0, \vt_1 \in \R^d$. We omit the temperature parameter $\exp(t)$, as it can be absorbed into the text embeddings. 

We examine how the parameter $\vw = [w_1, \cdots, w_d]^\top$ changes after adaptation. Since scaling $\vw$ by a constant does not affect the direction of the image embedding $\vz_i$ or the prediction, we analyze the following normalized quantity, referred to as the \textit{feature importance}:
\begin{align}
    r_h = \frac{|w_h|}{\sum_{l=1}^d |w_l|}.
\end{align}
A larger $r_h$ indicates that the $h$-th feature contributes more to the prediction. Ideally, class-relevant features should have larger ratios, while domain-specific ones should have smaller ratios to ensure robustness to distribution shifts. Before adaptation, $\vw = \vone$, so all features have equal importance, i.e., $r_h = 1/d, \forall h$. The following Theorem \ref{thm:tta} analyzes how adaptation changes these feature importance.

\tta*  

\begin{proof}
    In the forward pass, for each image $j \in \sS$, its logits $\vl_j = [l_{j0}, l_{j1}]^\top$, predicted probabilities $\vp_j = [p_{j0}, p_{j1}]^\top$, and entropy $H_j$ are given by
    \begin{align*}
        l_{jc} &= \vz_j^\top \vt_c, \quad c \in \{0,1\}, \\
        p_{jc} &= \frac{\exp(l_{jc})}{\exp(l_{j0}) + \exp(l_{j1})}, \quad c \in \{0,1\}, \\
        H_j &= -\sum_{c \in \{0,1\}} p_{jc} \log p_{jc}. 
    \end{align*}
    And the total entropy on the support set $\sS$ is
    \begin{align*}
        H &= \E_{j \in \sS} H_i . 
    \end{align*}
    We then compute the back-propagation. For each sample $j$, 
    \begin{align*}
        \frac{\partial H_j}{\partial l_{j0}} 
        &= \frac{\partial H_j}{\partial p_{j0}} \cdot \frac{\partial p_{j0}}{\partial l_{j0}} 
        + \frac{\partial H_j}{\partial p_{j1}} \cdot \frac{\partial p_{j1}}{\partial l_{j0}} \\
        &= - (\log p_{j0} + 1) \cdot (p_{j0} p_{j1}) 
        - (\log p_{j1} + 1) \cdot (-p_{j0} p_{j1}) \\
        &= (\log p_{j1} - \log p_{j0}) \cdot (p_{j0} p_{j1}) \\
        &= (l_{j1} - l_{j0}) \cdot (p_{j0} p_{j1}). 
    \end{align*}
    By symmetry, 
    \begin{align*}
        \frac{\partial H_j}{\partial l_{j1}} 
        = (l_{j0} - l_{j1}) \cdot (p_{j0} p_{j1}).
    \end{align*}
    The gradient w.r.t. $\vz_j$ is
    \begin{align*}
        \frac{\partial H_j}{\partial \vz_j} 
        &= \frac{\partial H_j}{\partial l_{j0}} \cdot \frac{\partial l_{j0}}{\partial \vz_j} 
         + \frac{\partial H_j}{\partial l_{j1}} \cdot \frac{\partial l_{j1}}{\partial \vz_j} \\
        &= (l_{j1} - l_{j0}) \cdot (p_{j0} p_{j1}) \cdot \vt_0 
         + (l_{j0} - l_{j1}) \cdot (p_{j0} p_{j1}) \cdot \vt_1 \\
        &= - (l_{j1} - l_{j0}) \cdot (p_{j0} p_{j1}) \cdot (\vt_1 - \vt_0) 
    \end{align*}
    Therefore, the gradient w.r.t. $\vw$ is 
    \begin{align*}
        \frac{\partial H_j}{\partial \vw} 
        &= \frac{\partial \vz_j}{\partial \vw} \frac{\partial H_j}{\partial \vz_j} \\
        &= \frac{\partial H_j}{\partial \vz_j} \odot \vv_j \\
        &= - (l_{j1} - l_{j0}) \cdot (p_{j0} p_{j1}) \cdot (\vt_1 - \vt_0) \odot  \vv_j \\
        &= - (p_{j0} p_{j1}) \cdot (\vt_1 - \vt_0) \odot  \vv_j (\vz_j^\top \vt_1 - \vz_j^\top \vt_0) \\
        &= - (p_{j0} p_{j1}) \cdot (\vt_1 - \vt_0) \odot  \vv_j (\vv_j^\top \vt_1 - \vv_j^\top \vt_0) \tag{$\vw = \vone, \vb = \vzero$}\\
        &= - (p_{j0} p_{j1}) \cdot (\vt_1 - \vt_0) \odot  \vv_j \vv_j^\top (\vt_1 - \vt_0) \\
        &= - (p_{j0} p_{j1}) \cdot \diag(\vt_1 - \vt_0) \vv_j \vv_j^\top (\vt_1 - \vt_0) \\
        &= - \diag(\vt_1 - \vt_0) \left[ (p_{j0} p_{j1}) \vv_j \vv_j^\top \right](\vt_1 - \vt_0) \\
        \frac{\partial H}{\partial \vw} &= \E_{j \in \sS} \frac{\partial H_j}{\partial \vw} \\
        &= - \diag(\vt_1 - \vt_0) \E_{j \in \sS}\left[ (p_{j0} p_{j1}) \vv_j \vv_j^\top \right](\vt_1 - \vt_0) 
    \end{align*}
    Therefore, when we conduct gradient descent with learning rate $\eta$, 
    \begin{align*}
        \vw \leftarrow \vone + \eta \diag(\vt_1 - \vt_0) \E_{j \in \sS}\left[ (p_{j0} p_{j1}) \vv_j \vv_j^\top \right](\vt_1 - \vt_0). 
    \end{align*}
    Besides, the gradient w.r.t. $\vb$ is expected to be small: 
    \begin{align*}
        \frac{\partial H}{\partial \vb} = \E_{j \in \sS} \frac{\partial H_j}{\partial \vz_j} = - \left(\E_{j \in \sS} (l_{j1} - l_{j0}) \cdot (p_{j0} p_{j1}) \right) \cdot (\vt_1 - \vt_0). 
    \end{align*}
    When the retrieved samples are class-balanced, we have 
    \begin{align*}
        \E_{j \in \sS} (l_{j1} - l_{j0}) \cdot (p_{j0} p_{j1}) \approx 0,
    \end{align*}
    and thus the gradient w.r.t. bias becomes negligible.
\end{proof}
\newpage
\section{Additional Experiments}

\subsection{Experimental Setup} 
\label{appendix:exp:setup}

\paragraph{Compute resources}
All of our experiments are conducted on single NVIDIA Tesla V100 with 32GB memory, except for
experiments on large batch size are conducted on single NVIDIA Tesla A100 with 80GB memory. 

\paragraph{Datasets}
For image corruption benchmarks (CIFAR-10-C, CIFAR-100-C, ImageNet-C), we use the data provided in \cite{corruption}, following the implementation in \cite{vte}. For DomainNet \cite{domainnet}, we use the implementation provided by DomainBed \cite{domainbed}. 

\paragraph{Pre-trained models}
We use the pre-trained models provided in the original CLIP repository \cite{clip}. We use batch size $B=100$ for ViT-B/32 and ViT-B/16, and $B=50$ for ViT-L/14 to save GPU memory usage. 

\paragraph{Text embeddings}
For all methods, except from CLIPArTT \cite{clipartt} which modifies the prompts, we use the 7 template in \cite{tip-adaptor}: 
\begin{itemize}
    \item ``\texttt{itap of a \{class\}}''
    \item ``\texttt{a bad photo of the \{class\}}''
    \item ``\texttt{a origami \{class\}}''
    \item ``\texttt{a photo of the large \{class\}}''
    \item ``\texttt{a \{class\} in a video game}''  
    \item ``\texttt{art of the \{class\}}''
    \item ``\texttt{a photo of the small \{class\}}''
\end{itemize}
The text embedding for each class $y$ is computed by
\begin{align}
    \vt_y = \normalize\left( \sum_{\kappa = 1}^k \vt_{y, \kappa} \right), \text{where } \vt_{y, \kappa} = \normalize(\mathrm{text\_encoder(\{template_\kappa, classname_{\mathit{y}}\})})
\end{align}
As a reference, we report the zero-shot performance of CLIP using both a single template ``\texttt{a photo of a \{class\}.}'' and an ensemble of seven templates, denoted as CLIP and Ensemble, respectively.

\paragraph{Hyperparameters}
We rerun all baseline methods and perform separate hyperparameter searches for both the single-domain and mixed-domain settings, reporting the best configurations for each. This is particularly important for online TTA methods (Tent \cite{tent}, NOTE \cite{note}, SAR \cite{sar}), which suffer from batch dependency: their optimal learning rate is inversely correlated with the number of batches (i.e., update steps). Prior work \cite{sar} uses the same hyperparameters for both settings, which may underestimate the baselines’ performance under mixed-domain shifts. Our separate tuning ensures that all results are fairly compared and free of bias.The final hyperparameter choices are listed below. 
\begin{itemize}

    \item \textbf{Tent} \cite{tent}. On CIFAR-10-C/CIFAR-100-C/ImageNet-C/DomainNet (same order below), we use SignSGD optimizer, and lr = 2e-5/2e-5/1e-4/1e-6 (mixed-domain shift) or lr = 1e-4/2e-4/2e-3/1e-6 (single-domain shift). 
    
    \item \textbf{NOTE} \cite{note}. Memory size is set to be the same as batch size. We use SignSGD optimizer, and lr = 2e-5/2e-5/1e-4/2e-5 (mixed-domain shift) or lr = 1e-4/2e-4/2e-3/1e-4 (single-domain shift). 
    
    \item \textbf{SAR} \cite{sar}. We use SignSGD as the base optimizer of SAM, and lr = 2e-5/2e-5/1e-4/2e-5 (mixed-domain shift) or lr = 1e-4/2e-4/2e-3/1e-4 (single-domain shift). For all setting we use the same $E_0 = 0.4 \times \ln C$, where $C$ is the number of classes, $\rho = 0.05$, and $e_0 = 0.2$, following their original hyperparameters. 
    
    \item \textbf{RoTTA} \cite{rotta}. We use SignSGD optimizer, and lr = 1e-2/1e-2/1e-2/1e-5 (mixed-domain shift) or lr = 2e-2/2e-2/2e-2/1e-4 (single-domain shift). 

    \item \textbf{TDA} \cite{tda}. We use $\alpha$ = 1.0 for mixed-domain shift and $\alpha$ = 1.0/1.0/2.0/2.0 for single-domain shift. We use the default value for other hyperparameters. 

    \item \textbf{DMN} \cite{dmn}. We use $\alpha$ = 0.1 for mixed-domain shift and $\alpha$ = 0.2/1.0/1.0/0.2 for single-domain shift. 

    \item \textbf{WATT-S} \cite{watt}. We use Adam optimizer, lr = 1e-3, $L = 2$ and $M = 5$, provided by the original paper in the same datasets. 

    \item \textbf{CLIPArTT} \cite{clipartt}. We use Adam optimizer, lr = 1e-3, $K=3$ and Iter=10, provided by the original paper in the same datasets. 

    \item \textbf{Mint} \cite{mint}. We use Adam optimizer, lr = 5e-3/5e-3/1e-2/5e-3 (mixed-domain shift) or lr = 7e-3/7e-3/1.5e-2/5e-3 (single-domain shift), $K_{\text{prior}} = +\infty$ for mixed-domain shift and 10,000/10,000/10,000/100,000 for single-domain shifts. Single-domain hyperparameters follow those in the original paper. 

    \item \textbf{\algname}. We use the same set of hyperparameters for both mixed-domain and single-domain shifts to demonstrate the robustness of {\algname}. We set the maximum capacity $K$ =7,500/750/75/300 on CIFAR-10-C, CIFAR-100-C, ImageNet-C, and DomainNet, respectively. As a result, the total capacity $C\cdot K$ remains roughly consistent across datasets. We use $k$ =50/5/1/10, and $\beta$ = 5.0/5.0/0.0/5.0. The optimizer is SignSGD with learning rate 1e-2 of for all datasets.
    
\end{itemize}

\newpage
\subsection{Error Bars (RQ1)} 
\label{appendix:exp:error_bars}

Tables \ref{tab:corruption_meansd} and \ref{tab:domainnet_meansd} report the results in Tables \ref{tab:corruption} and \ref{tab:domainnet} with error bars, respectively. 

\renewcommand{\meansd}[2]{#1 \scriptsize(#2)}

\begin{table*}[h!]
    \centering
    \caption{Accuracy (mean (s.d.) \%) on corruption benchmarks under mixture of 15 corruptions. Best and second-best results are shown in \textbf{bold} and \underline{underlined}, respectively. }
    \label{tab:corruption_meansd}
    \resizebox{1.0\linewidth}{!}{
    \setlength{\tabcolsep}{3.0pt}{
        \begin{tabular}{llccccccccccccccc >{\columncolor{cvprblue!15}}c}
            \toprule
            & & \multicolumn{16}{c}{ViT-B/32 on CIFAR-10-C} \\
            \cmidrule(lr){3-18}
            Method & Venue & \multicolumn{3}{c}{Noise} & \multicolumn{4}{c}{Blur} & \multicolumn{4}{c}{Weather} & \multicolumn{4}{c}{Digital} & \multirow{2.5}{*}{Avg.} \cellcolor{white} \\
            \cmidrule(lr){3-5} \cmidrule(lr){6-9} \cmidrule(lr){10-13} \cmidrule(lr){14-17}
            & & Gauss. & Shot & Impul. 
            & Defoc. & Glass & Motion & Zoom  
            & Snow & Frost & Fog & Brit. 
            & Contr. & Elastic & Pixel & JPEG \\
            \midrule
            CLIP \cite{clip} & ICML'21 &
                35.5 & 40.0 & 43.2 & 70.0 & 41.4 & 64.5 & 70.2 & 70.8 & 72.3 & 66.7 & 81.4 & 64.5 & 59.6 & 48.2 & 56.7 & 59.0 \\
            Ensemble & - &
                38.6 & 42.6 & 42.7 & 72.4 & 43.9 & 66.6 & 71.6 & 73.8 & 75.7 & 69.0 & 83.6 & 67.0 & 61.9 & 51.8 & 58.6 & 61.3 \\
            Tent \cite{tent} & ICLR'21 & 
                \meansd{39.5}{0.6} & \meansd{43.1}{0.4} & \meansd{46.4}{0.3} & \meansd{\underline{77.8}}{0.3} & \meansd{56.6}{0.3} & \meansd{72.7}{0.4} & \meansd{77.9}{0.4} & \meansd{\underline{79.7}}{0.2} & \meansd{79.9}{0.2} & \meansd{74.0}{0.3} & \meansd{86.9}{0.1} & \meansd{75.5}{0.1} & \meansd{69.7}{0.3} & \meansd{60.3}{0.1} & \meansd{63.9}{0.4} & \meansd{66.9}{0.2} \\
            NOTE \cite{note} & NeurIPS'22 & 
                \meansd{44.9}{4.4} & \meansd{48.3}{3.9} & \meansd{48.9}{1.7} & \meansd{77.1}{0.5} & \meansd{57.0}{0.7} & \meansd{74.0}{1.4} & \meansd{77.6}{0.7} & \meansd{78.0}{0.8} & \meansd{79.2}{0.9} & \meansd{73.0}{1.4} & \meansd{86.5}{0.5} & \meansd{73.5}{1.8} & \meansd{68.7}{1.1} & \meansd{57.9}{1.5} & \meansd{62.2}{0.7} & \meansd{67.1}{0.8} \\
            SAR \cite{sar} & ICLR'23 & 
                \meansd{48.8}{0.7} & \meansd{51.8}{0.6} & \meansd{51.0}{0.5} & \meansd{77.5}{0.2} & \meansd{60.5}{0.2} & \meansd{74.8}{0.5} & \meansd{\textbf{79.2}}{0.3} & \meansd{79.6}{0.2} & \meansd{\textbf{80.7}}{0.4} & \meansd{76.6}{0.4} & \meansd{\underline{87.4}}{0.1} & \meansd{77.0}{0.2} & \meansd{71.0}{0.3} & \meansd{57.5}{0.4} & \meansd{62.8}{0.3} & \meansd{69.1}{0.2} \\
            RoTTA \cite{rotta} & CVPR'23 & 
                \meansd{\underline{53.1}}{1.3} & \meansd{\underline{56.4}}{1.4} & \meansd{\underline{51.3}}{0.9} & \meansd{\textbf{78.1}}{0.5} & \meansd{\underline{61.0}}{1.0} & \meansd{75.2}{0.2} & \meansd{\underline{79.0}}{0.4} & \meansd{79.0}{0.6} & \meansd{80.1}{1.0} & \meansd{\textbf{80.0}}{0.4} & \meansd{85.6}{0.5} & \meansd{\textbf{82.9}}{0.4} & \meansd{\textbf{73.7}}{0.9} & \meansd{\underline{65.3}}{0.9} & \meansd{\textbf{69.9}}{0.7} & \meansd{\underline{71.4}}{0.4} \\
            TDA \cite{tda} & CVPR'24 & 
                \meansd{39.8}{0.4} & \meansd{42.8}{0.4} & \meansd{43.4}{0.3} & \meansd{73.5}{0.3} & \meansd{45.4}{0.2} & \meansd{67.6}{0.2} & \meansd{73.2}{0.2} & \meansd{74.3}{0.2} & \meansd{76.4}{0.3} & \meansd{69.8}{0.2} & \meansd{84.0}{0.2} & \meansd{66.6}{0.2} & \meansd{62.5}{0.1} & \meansd{52.1}{0.3} & \meansd{57.7}{0.4} & \meansd{61.9}{0.1} \\
            DMN-ZS \cite{dmn} & CVPR'24 & 
                \meansd{38.4}{0.1} & \meansd{41.7}{0.1} & \meansd{42.4}{0.1} & \meansd{73.1}{0.0} & \meansd{44.6}{0.0} & \meansd{67.2}{0.1} & \meansd{72.8}{0.1} & \meansd{74.6}{0.1} & \meansd{76.3}{0.0} & \meansd{69.5}{0.0} & \meansd{84.0}{0.0} & \meansd{66.5}{0.0} & \meansd{62.3}{0.0} & \meansd{52.4}{0.1} & \meansd{58.3}{0.1} & \meansd{61.6}{0.0} \\
            WATT-S \cite{watt} & NeurIPS'24 & 
                \meansd{47.4}{0.3} & \meansd{49.7}{0.1} & \meansd{49.5}{0.3} & \meansd{77.0}{0.2} & \meansd{54.3}{0.2} & \meansd{72.9}{0.2} & \meansd{77.3}{0.2} & \meansd{77.8}{0.2} & \meansd{79.5}{0.1} & \meansd{74.8}{0.2} & \meansd{86.4}{0.1} & \meansd{74.1}{0.2} & \meansd{67.8}{0.2} & \meansd{57.2}{0.4} & \meansd{62.9}{0.2} & \meansd{67.2}{0.0} \\
            CLIPArTT \cite{clipartt} & WACV'25 & 
                \meansd{37.2}{0.2} & \meansd{41.5}{0.2} & \meansd{44.8}{0.2} & \meansd{70.5}{0.1} & \meansd{48.9}{0.2} & \meansd{65.5}{0.2} & \meansd{70.7}{0.3} & \meansd{72.9}{0.3} & \meansd{75.3}{0.1} & \meansd{67.4}{0.2} & \meansd{82.8}{0.3} & \meansd{66.9}{0.1} & \meansd{61.9}{0.3} & \meansd{58.5}{0.2} & \meansd{59.5}{0.3} & \meansd{61.6}{0.1} \\
            Mint \cite{mint} & NeurIPS'25 & 
                \meansd{47.9}{0.1} & \meansd{50.9}{0.1} & \meansd{50.3}{0.1} & \meansd{74.1}{0.0} & \meansd{51.3}{0.1} & \meansd{\underline{75.8}}{0.0} & \meansd{76.9}{0.0} & \meansd{76.4}{0.1} & \meansd{77.0}{0.1} & \meansd{73.8}{0.1} & \meansd{85.8}{0.1} & \meansd{73.8}{0.0} & \meansd{65.8}{0.1} & \meansd{47.3}{0.0} & \meansd{56.4}{0.1} & \meansd{65.6}{0.0} \\
            {\algname} & - & 
                \meansd{\textbf{61.0}}{0.4} & \meansd{\textbf{63.3}}{0.3} & \meansd{\textbf{56.0}}{0.2} & \meansd{77.2}{0.1} & \meansd{\textbf{64.1}}{0.2} & \meansd{\textbf{77.2}}{0.1} & \meansd{78.9}{0.3} & \meansd{\textbf{80.9}}{0.1} & \meansd{\underline{80.3}}{0.2} & \meansd{\underline{78.0}}{0.2} & \meansd{\textbf{88.1}}{0.2} & \meansd{\underline{79.8}}{0.2} & \meansd{\underline{72.0}}{0.1} & \meansd{\textbf{65.5}}{0.1} & \meansd{\underline{67.5}}{0.2} & \meansd{\textbf{72.7}}{0.1} \\
            \midrule 
            & & \multicolumn{16}{c}{ViT-B/16 on CIFAR-100-C} \\
            \cmidrule(lr){3-18}
            Method & Venue & \multicolumn{3}{c}{Noise} & \multicolumn{4}{c}{Blur} & \multicolumn{4}{c}{Weather} & \multicolumn{4}{c}{Digital} & \multirow{2.5}{*}{Avg.} \cellcolor{white}\\
            \cmidrule(lr){3-5} \cmidrule(lr){6-9} \cmidrule(lr){10-13} \cmidrule(lr){14-17}
            & & Gauss. & Shot & Impul. 
            & Defoc. & Glass & Motion & Zoom  
            & Snow & Frost & Fog & Brit. 
            & Contr. & Elastic & Pixel & JPEG \\
            \midrule
            CLIP \cite{clip} & ICML'21 &
                19.7 & 21.4 & 25.3 & 42.5 & 20.2 & 43.1 & 48.0 & 48.4 & 49.7 & 41.7 & 57.0 & 34.5 & 29.2 & 23.9 & 32.4 & 35.8 \\
            Ensemble & - &
                22.9 & 24.4 & 29.6 & 43.6 & 20.1 & 43.7 & 48.7 & 48.9 & 50.4 & 41.8 & 58.1 & 35.3 & 29.2 & 26.3 & 33.6 & 37.1 \\
            Tent \cite{tent} & ICLR'21 & 
                \meansd{24.9}{0.2} & \meansd{26.7}{0.2} & \meansd{34.4}{0.1} & \meansd{49.7}{0.1} & \meansd{23.9}{0.2} & \meansd{47.5}{0.1} & \meansd{53.9}{0.1} & \meansd{52.9}{0.1} & \meansd{51.4}{0.2} & \meansd{45.3}{0.2} & \meansd{62.9}{0.2} & \meansd{43.6}{0.1} & \meansd{32.0}{0.1} & \meansd{\underline{31.7}}{0.1} & \meansd{37.1}{0.1} & \meansd{41.2}{0.0} \\
            NOTE \cite{note} & NeurIPS'22 & 
                \meansd{25.7}{1.4} & \meansd{27.3}{1.8} & \meansd{35.4}{1.2} & \meansd{49.9}{0.7} & \meansd{23.9}{1.2} & \meansd{47.6}{0.7} & \meansd{54.2}{1.1} & \meansd{52.7}{0.7} & \meansd{51.8}{1.2} & \meansd{45.9}{0.5} & \meansd{63.7}{0.9} & \meansd{44.0}{1.1} & \meansd{32.4}{1.4} & \meansd{30.2}{2.1} & \meansd{36.0}{0.8} & \meansd{41.4}{0.8} \\
            SAR \cite{sar} & ICLR'23 & 
                \meansd{\underline{28.4}}{0.2} & \meansd{\underline{30.5}}{0.3} & \meansd{\underline{38.5}}{0.2} & \meansd{\underline{50.4}}{0.2} & \meansd{24.6}{0.4} & \meansd{\underline{49.2}}{0.2} & \meansd{\underline{55.1}}{0.2} & \meansd{54.1}{0.2} & \meansd{53.4}{0.2} & \meansd{47.2}{0.2} & \meansd{\underline{64.0}}{0.3} & \meansd{45.0}{0.2} & \meansd{33.6}{0.1} & \meansd{29.7}{0.2} & \meansd{37.4}{0.2} & \meansd{\underline{42.7}}{0.1} \\
            RoTTA \cite{rotta} & CVPR'23 & 
                \meansd{22.9}{0.6} & \meansd{24.5}{0.8} & \meansd{32.4}{1.0} & \meansd{47.9}{1.2} & \meansd{\underline{25.4}}{1.2} & \meansd{46.5}{1.2} & \meansd{50.1}{0.9} & \meansd{50.5}{1.4} & \meansd{50.7}{1.1} & \meansd{\underline{50.8}}{0.7} & \meansd{58.3}{1.0} & \meansd{\textbf{53.1}}{0.7} & \meansd{\textbf{37.7}}{0.8} & \meansd{29.9}{1.0} & \meansd{36.2}{0.8} & \meansd{41.1}{0.8} \\
            TDA \cite{tda} & CVPR'24 & 
                \meansd{23.4}{0.1} & \meansd{25.2}{0.2} & \meansd{30.3}{0.2} & \meansd{44.1}{0.2} & \meansd{20.3}{0.2} & \meansd{43.8}{0.1} & \meansd{49.1}{0.2} & \meansd{49.3}{0.2} & \meansd{51.3}{0.1} & \meansd{42.2}{0.1} & \meansd{58.7}{0.1} & \meansd{36.1}{0.1} & \meansd{29.3}{0.1} & \meansd{26.4}{0.2} & \meansd{33.6}{0.1} & \meansd{37.5}{0.0} \\
            DMN-ZS \cite{dmn} & CVPR'24 & 
                \meansd{22.9}{0.1} & \meansd{24.4}{0.2} & \meansd{29.6}{0.1} & \meansd{44.3}{0.1} & \meansd{20.2}{0.1} & \meansd{44.1}{0.0} & \meansd{49.4}{0.1} & \meansd{49.7}{0.1} & \meansd{51.1}{0.1} & \meansd{42.2}{0.1} & \meansd{58.8}{0.1} & \meansd{35.5}{0.1} & \meansd{29.4}{0.0} & \meansd{26.2}{0.1} & \meansd{34.0}{0.1} & \meansd{37.5}{0.0} \\
            WATT-S \cite{watt} & NeurIPS'24 & 
                \meansd{26.4}{0.2} & \meansd{28.4}{0.3} & \meansd{35.6}{0.3} & \meansd{50.2}{0.3} & \meansd{24.2}{0.3} & \meansd{48.5}{0.2} & \meansd{54.4}{0.2} & \meansd{\underline{54.2}}{0.3} & \meansd{\underline{54.0}}{0.2} & \meansd{47.7}{0.3} & \meansd{63.5}{0.4} & \meansd{43.3}{0.2} & \meansd{34.0}{0.2} & \meansd{31.4}{0.2} & \meansd{\underline{37.7}}{0.4} & \meansd{42.2}{0.1} \\
            CLIPArTT \cite{clipartt} & WACV'25 & 
                \meansd{21.1}{0.3} & \meansd{22.8}{0.2} & \meansd{29.4}{0.2} & \meansd{46.3}{0.1} & \meansd{24.5}{0.3} & \meansd{45.0}{0.1} & \meansd{51.1}{0.1} & \meansd{51.5}{0.1} & \meansd{51.8}{0.2} & \meansd{44.0}{0.3} & \meansd{61.1}{0.2} & \meansd{39.0}{0.2} & \meansd{33.2}{0.2} & \meansd{28.3}{0.3} & \meansd{36.7}{0.1} & \meansd{39.0}{0.1} \\
            Mint \cite{mint} & NeurIPS'25 & 
                \meansd{22.3}{0.0} & \meansd{24.5}{0.1} & \meansd{33.0}{0.1} & \meansd{49.2}{0.1} & \meansd{22.4}{0.1} & \meansd{47.8}{0.1} & \meansd{54.2}{0.1} & \meansd{52.4}{0.0} & \meansd{51.6}{0.1} & \meansd{47.7}{0.1} & \meansd{63.7}{0.1} & \meansd{42.2}{0.1} & \meansd{31.8}{0.0} & \meansd{24.0}{0.1} & \meansd{34.3}{0.0} & \meansd{40.1}{0.0} \\
            {\algname} & - & 
                \meansd{\textbf{30.3}}{0.3} & \meansd{\textbf{32.9}}{0.4} & \meansd{\textbf{42.8}}{0.3} & \meansd{\textbf{52.5}}{0.3} & \meansd{\textbf{29.1}}{0.4} & \meansd{\textbf{52.2}}{0.2} & \meansd{\textbf{55.5}}{0.1} & \meansd{\textbf{55.4}}{0.1} & \meansd{\textbf{55.1}}{0.2} & \meansd{\textbf{52.6}}{0.2} & \meansd{\textbf{65.9}}{0.2} & \meansd{\underline{51.0}}{0.2} & \meansd{\underline{36.8}}{0.2} & \meansd{\textbf{40.4}}{0.4} & \meansd{\textbf{39.6}}{0.2} & \meansd{\textbf{46.1}}{0.1} \\
            \midrule 
            & & \multicolumn{16}{c}{ViT-L/14 on ImageNet-C} \\
            \cmidrule(lr){3-18}
            Method & Venue & \multicolumn{3}{c}{Noise} & \multicolumn{4}{c}{Blur} & \multicolumn{4}{c}{Weather} & \multicolumn{4}{c}{Digital} & \multirow{2.5}{*}{Avg.} \cellcolor{white}\\
            \cmidrule(lr){3-5} \cmidrule(lr){6-9} \cmidrule(lr){10-13} \cmidrule(lr){14-17}
            & & Gauss. & Shot & Impul. 
            & Defoc. & Glass & Motion & Zoom  
            & Snow & Frost & Fog & Brit. 
            & Contr. & Elastic & Pixel & JPEG \\
            \midrule
            CLIP \cite{clip} & ICML'21 &
                27.4 & 29.4 & 28.7 & 34.6 & 25.3 & 41.0 & 36.7 & 49.8 & 44.1 & 49.7 & 65.4 & 35.1 & 30.3 & 53.5 & 42.2 & 39.6 \\
            Ensemble & - &
                29.1 & 30.4 & 30.1 & 37.5 & 27.3 & 44.2 & 39.2 & 52.4 & 46.4 & 52.6 & 67.8 & 34.4 & 32.4 & 56.2 & 44.2 & 41.6 \\
            Tent \cite{tent} & ICLR'21 & 
                \meansd{36.3}{0.4} & \meansd{38.0}{0.1} & \meansd{38.5}{0.1} & \meansd{\underline{40.3}}{0.3} & \meansd{37.4}{0.5} & \meansd{\underline{47.7}}{0.3} & \meansd{\underline{43.5}}{0.4} & \meansd{54.2}{0.4} & \meansd{49.5}{0.2} & \meansd{56.4}{0.3} & \meansd{67.8}{0.2} & \meansd{45.4}{0.7} & \meansd{40.6}{0.5} & \meansd{57.6}{0.3} & \meansd{46.6}{0.3} & \meansd{\underline{46.7}}{0.1} \\
            NOTE \cite{note} & NeurIPS'22 & 
                \meansd{36.3}{0.4} & \meansd{38.0}{0.2} & \meansd{38.6}{0.3} & \meansd{40.0}{0.4} & \meansd{37.3}{0.4} & \meansd{47.6}{0.3} & \meansd{\underline{43.5}}{0.3} & \meansd{\underline{54.4}}{0.4} & \meansd{49.7}{0.5} & \meansd{56.4}{0.5} & \meansd{67.8}{0.3} & \meansd{44.5}{0.4} & \meansd{40.8}{0.4} & \meansd{57.5}{0.3} & \meansd{46.2}{0.5} & \meansd{46.6}{0.1} \\
            SAR \cite{sar} & ICLR'23 & 
                \meansd{36.3}{0.8} & \meansd{38.3}{0.4} & \meansd{38.5}{0.4} & \meansd{39.2}{0.7} & \meansd{\underline{38.5}}{0.8} & \meansd{46.8}{0.5} & \meansd{43.4}{0.6} & \meansd{53.9}{0.3} & \meansd{\underline{50.4}}{0.5} & \meansd{\underline{56.6}}{0.4} & \meansd{66.8}{0.2} & \meansd{\underline{45.5}}{0.1} & \meansd{\underline{41.8}}{0.8} & \meansd{56.2}{0.6} & \meansd{46.6}{0.5} & \meansd{46.6}{0.3} \\
            RoTTA \cite{rotta} & CVPR'23 & 
                \meansd{\underline{37.4}}{0.5} & \meansd{\underline{38.7}}{0.4} & \meansd{\underline{39.6}}{0.7} & \meansd{36.9}{0.5} & \meansd{34.5}{0.5} & \meansd{44.0}{0.5} & \meansd{39.3}{0.7} & \meansd{53.5}{0.6} & \meansd{48.6}{0.6} & \meansd{54.2}{0.7} & \meansd{67.2}{0.3} & \meansd{43.0}{0.8} & \meansd{39.7}{0.6} & \meansd{56.3}{0.5} & \meansd{\underline{48.4}}{0.6} & \meansd{45.4}{0.4} \\
            TDA \cite{tda} & CVPR'24 & 
                \meansd{29.6}{0.1} & \meansd{31.0}{0.1} & \meansd{31.5}{0.2} & \meansd{38.1}{0.1} & \meansd{28.9}{0.1} & \meansd{44.6}{0.1} & \meansd{40.0}{0.1} & \meansd{53.4}{0.1} & \meansd{47.5}{0.1} & \meansd{53.3}{0.2} & \meansd{\underline{68.5}}{0.1} & \meansd{39.1}{0.2} & \meansd{33.4}{0.2} & \meansd{57.1}{0.2} & \meansd{45.1}{0.2} & \meansd{42.7}{0.0} \\
            DMN-ZS \cite{dmn} & CVPR'24 & 
                \meansd{29.2}{0.1} & \meansd{30.5}{0.1} & \meansd{30.2}{0.0} & \meansd{37.6}{0.1} & \meansd{27.6}{0.1} & \meansd{44.4}{0.1} & \meansd{39.2}{0.1} & \meansd{52.5}{0.0} & \meansd{46.6}{0.1} & \meansd{52.7}{0.0} & \meansd{67.9}{0.0} & \meansd{33.8}{0.1} & \meansd{32.5}{0.0} & \meansd{56.5}{0.1} & \meansd{44.5}{0.1} & \meansd{41.7}{0.0} \\
            WATT-S \cite{watt} & NeurIPS'24 & 
                \meansd{31.8}{0.1} & \meansd{33.1}{0.1} & \meansd{33.6}{0.2} & \meansd{39.2}{0.2} & \meansd{30.7}{0.3} & \meansd{45.7}{0.1} & \meansd{41.3}{0.1} & \meansd{53.5}{0.2} & \meansd{47.3}{0.3} & \meansd{53.7}{0.2} & \meansd{67.9}{0.3} & \meansd{40.1}{0.1} & \meansd{34.7}{0.3} & \meansd{57.3}{0.1} & \meansd{45.6}{0.3} & \meansd{43.7}{0.0} \\
            CLIPArTT \cite{clipartt} & WACV'25 &
                \meansd{28.2}{0.3} & \meansd{30.2}{0.3} & \meansd{29.8}{0.2} & \meansd{35.0}{0.2} & \meansd{26.5}{0.3} & \meansd{41.4}{0.2} & \meansd{37.3}{0.1} & \meansd{50.3}{0.3} & \meansd{43.7}{0.3} & \meansd{49.7}{0.2} & \meansd{65.0}{0.2} & \meansd{38.5}{0.2} & \meansd{31.0}{0.3} & \meansd{53.5}{0.2} & \meansd{42.3}{0.3} & \meansd{40.2}{0.1} \\
            Mint \cite{mint} & NeurIPS'25 & 
                \meansd{33.8}{0.2} & \meansd{35.2}{0.2} & \meansd{35.8}{0.2} & \meansd{39.3}{0.2} & \meansd{33.3}{0.3} & \meansd{46.2}{0.2} & \meansd{41.5}{0.2} & \meansd{53.6}{0.3} & \meansd{47.4}{0.2} & \meansd{53.3}{0.1} & \meansd{67.7}{0.1} & \meansd{36.3}{0.5} & \meansd{38.2}{0.3} & \meansd{\underline{58.1}}{0.1} & \meansd{47.4}{0.7} & \meansd{44.5}{0.1} \\
            {\algname} & - & 
                \meansd{\textbf{37.5}}{0.2} & \meansd{\textbf{38.8}}{0.1} & \meansd{\textbf{41.1}}{0.2} & \meansd{\textbf{42.2}}{0.3} & \meansd{\textbf{39.6}}{0.4} & \meansd{\textbf{49.9}}{0.4} & \meansd{\textbf{46.2}}{0.4} & \meansd{\textbf{57.3}}{0.3} & \meansd{\textbf{51.6}}{0.1} & \meansd{\textbf{59.4}}{0.1} & \meansd{\textbf{68.7}}{0.2} & \meansd{\textbf{46.4}}{0.4} & \meansd{\textbf{45.2}}{0.2} & \meansd{\textbf{59.4}}{0.2} & \meansd{\textbf{54.7}}{0.2} & \meansd{\textbf{49.2}}{0.1} \\
            \bottomrule 
        \end{tabular}
    }
    }
\end{table*}
\begin{table}[h!]
    \centering
    \caption{Mean accuracy (\%) on DomainNet under mixture of six domains. Best and second-best results are shown in \textbf{bold} and \underline{underlined}, respectively. }
    \label{tab:domainnet_meansd}
    \resizebox{0.55\linewidth}{!}{
    \setlength{\tabcolsep}{3pt}{
        \begin{tabular}{llcccccc >{\columncolor{cvprblue!15}}c}
            \toprule
            \multirow{2.5}{*}{Method} & \multirow{2.5}{*}{Venue} & \multicolumn{7}{c}{ViT-B/32 on DomainNet} \\
            \cmidrule(lr){3-9}
            & & clip & info & paint & quick & real & sketch & Avg. \cellcolor{white} \\
            \midrule
            CLIP \cite{clip} & ICML'21 & 
                67.6 & 41.5 & 62.7 & 12.8 & 81.0 & 58.1 & 54.0 \\
            Ensemble & - & 
                68.9 & 44.2 & 64.7 & 13.3 & 82.3 & 60.3 & 55.6 \\
            Tent \cite{tent} & ICLR'21 &
                \meansd{69.1}{0.0} & \meansd{44.3}{0.0} & \meansd{64.9}{0.0} & \meansd{13.0}{0.0} & \meansd{82.3}{0.0} & \meansd{60.3}{0.0} & \meansd{55.7}{0.0} \\
            NOTE \cite{note} & NeurIPS'22 &
                \meansd{\underline{69.6}}{0.1} & \meansd{\underline{45.1}}{0.0} & \meansd{65.0}{0.0} & \meansd{16.4}{0.1} & \meansd{82.1}{0.0} & \meansd{\underline{60.8}}{0.0} & \meansd{\underline{56.5}}{0.0} \\
            SAR \cite{sar} & ICLR'23 & 
                \meansd{69.4}{0.0} & \meansd{44.7}{0.0} & \meansd{65.2}{0.1} & \meansd{14.5}{0.0} & \meansd{\underline{82.5}}{0.0} & \meansd{\underline{60.8}}{0.0} & \meansd{56.2}{0.0} \\
            RoTTA \cite{rotta} & CVPR'23 & 
                \meansd{69.4}{0.1} & \meansd{44.8}{0.0} & \meansd{\underline{65.3}}{0.0} & \meansd{14.3}{0.0} & \meansd{82.4}{0.0} & \meansd{\underline{60.8}}{0.0} & \meansd{56.2}{0.0} \\
            TDA \cite{tda} & CVPR'24 & 
                \meansd{68.9}{0.0} & \meansd{44.7}{0.0} & \meansd{65.1}{0.1} & \meansd{14.6}{0.1} & \meansd{82.3}{0.0} & \meansd{60.5}{0.0} & \meansd{56.0}{0.0} \\
            DMN-ZS \cite{dmn} & CVPR'24 & 
                \meansd{68.9}{0.0} & \meansd{44.6}{0.0} & \meansd{64.8}{0.0} & \meansd{12.7}{0.0} & \meansd{\textbf{82.6}}{0.0} & \meansd{60.2}{0.0} & \meansd{55.6}{0.0} \\
            WATT-S \cite{watt} & NeurIPS'24 & 
                \meansd{68.3}{0.0} & \meansd{42.8}{0.1} & \meansd{63.5}{0.1} & \meansd{\underline{17.0}}{0.0} & \meansd{81.4}{0.0} & \meansd{59.9}{0.1} & \meansd{55.5}{0.0} \\
            CLIPArTT \cite{clipartt} & WACV'25 & 
                \meansd{67.2}{0.0} & \meansd{41.9}{0.0} & \meansd{62.9}{0.0} & \meansd{12.6}{0.0} & \meansd{80.8}{0.0} & \meansd{57.8}{0.1} & \meansd{53.9}{0.0} \\
            Mint \cite{mint} & NeurIPS'25 & 
                \meansd{69.3}{0.0} & \meansd{44.4}{0.0} & \meansd{65.0}{0.0} & \meansd{15.0}{0.0} & \meansd{82.3}{0.0} & \meansd{60.5}{0.0} & \meansd{56.1}{0.0} \\
            {\algname} & - & 
                \meansd{\textbf{70.1}}{0.0} & \meansd{\textbf{45.2}}{0.0} & \meansd{\textbf{65.8}}{0.0} & \meansd{\textbf{17.6}}{0.1} & \meansd{82.4}{0.0} & \meansd{\textbf{61.5}}{0.0} & \meansd{\textbf{57.1}}{0.0} \\
            \bottomrule 
        \end{tabular}
    }
    }
\end{table}
\newpage
\subsection{Comparison of Single-Domain and Mixed-Domain Shifts (RQ1)} 
\label{appendix:exp:single_vs_mixed}

We compare the performance of all methods across both single-domain and mixed-domain settings on all datasets.
In the single-domain evaluation, models are tested on each domain separately without mixing test samples, and we report the average accuracy across all domains.
For all baselines, especially online adaptation methods, we tune the learning rate separately for the two settings, since their performance is highly sensitive to the product of update steps and learning rate.
In contrast, {\algname} uses the same hyperparameters across both settings.

We observe that most existing methods experience a noticeable performance drop under mixed-domain shifts compared to the single-domain case, with a few exceptions.
\begin{itemize}
    \item RoTTA, designed for continual TTA, shows limited degradation under domain mixing.
    \item On DomainNet, since different domains have distinct label distributions, mixing them unexpectedly reduces prediction bias for certain methods such as Tent and SAR. As a validation, NOTE, which already includes a prediction-balanced memory, still suffers under domain mixtures.
\end{itemize}
Across all four datasets, {\algname} maintains consistently strong performance in both single-domain and mixed-domain settings.

\begin{figure}[h!]
    \centering
    \begin{subfigure}[t]{0.49\linewidth}
        \centering
        \includegraphics[width=\linewidth]{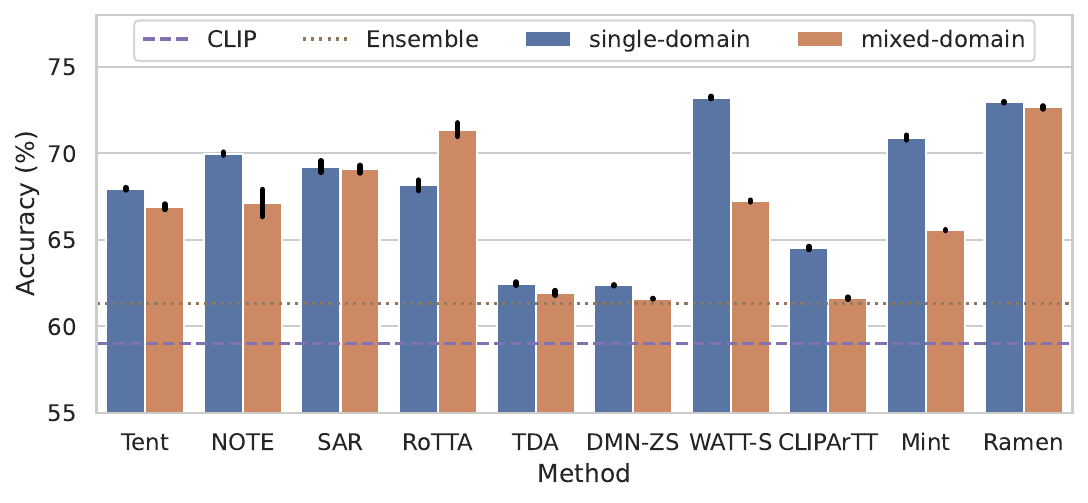}
        \subcaption{CIFAR-10-C}
    \end{subfigure}
    \hfill
    \begin{subfigure}[t]{0.49\linewidth}
        \centering
        \includegraphics[width=\linewidth]{figure/cifar100_compare.pdf}
        \subcaption{CIFAR-100-C}
    \end{subfigure}
    
    \begin{subfigure}[t]{0.49\linewidth}
        \centering
        \includegraphics[width=\linewidth]{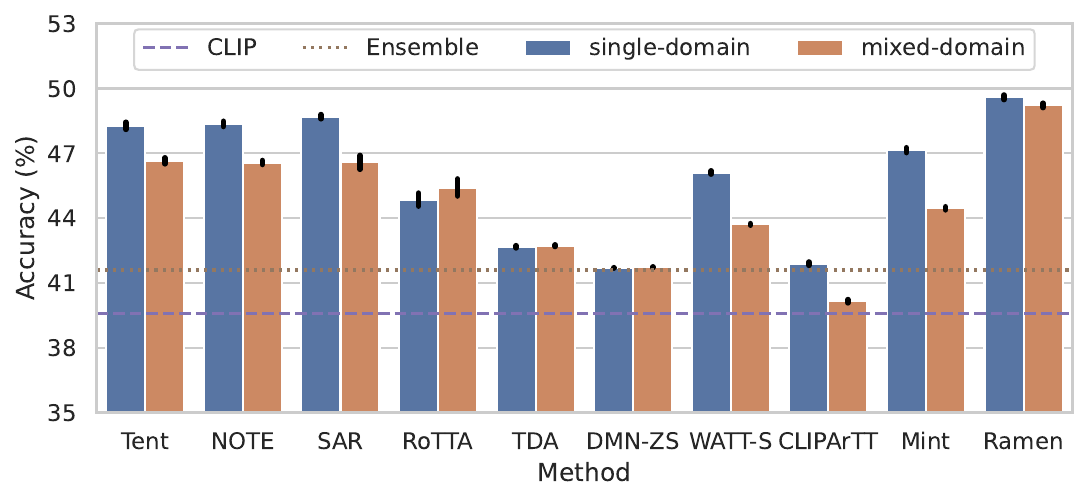}
        \subcaption{ImageNet-C}
    \end{subfigure}
    \hfill
    \begin{subfigure}[t]{0.49\linewidth}
        \centering
        \includegraphics[width=\linewidth]{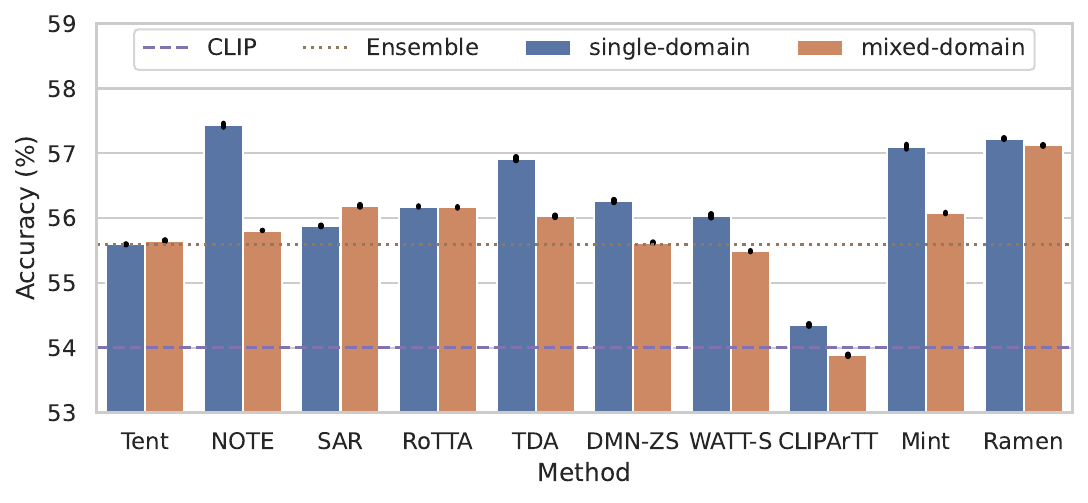}
        \subcaption{DomainNet}
    \end{subfigure}
    \caption{Performance comparison of TTA methods under single-domain and mixed-domain shifts.}
    \label{fig:single_vs_mixed:all}
\end{figure}
\newpage
\subsection{Visualization of Active Sample Selection (RQ2)} 
\label{appendix:exp:vis_mat}

We visualize, for all datasets, the domain distribution of retrieved samples during active sample selection. Each entry $(i, j)$ indicates the average proportion (\%) of support samples from domain $j$ when the test sample comes from domain $i$. Numbers in parentheses denote the proportion of support samples that come from the same domain as the test sample. On average, 44.5\%/40.9\%/38.0\%/47.8\% of the retrieved samples come from the same domain as the test sample on CIFAR-10-C/CIFAR-100-C/ImageNet-C/DomainNet, respectively, which is substantially higher than random selection (6.7\%/6.7\%/6.7\%/16.7\%). Moreover, even when the retrieved samples are from different domains, they typically belong to semantically or visually related domains (e.g., Gaussian noise and shot noise, motion blur and zoom blur, sketch and quickdraw).

\begin{figure}[h!]
    \centering
    \begin{subfigure}[t]{0.49\linewidth}
        \centering
        \includegraphics[width=\linewidth]{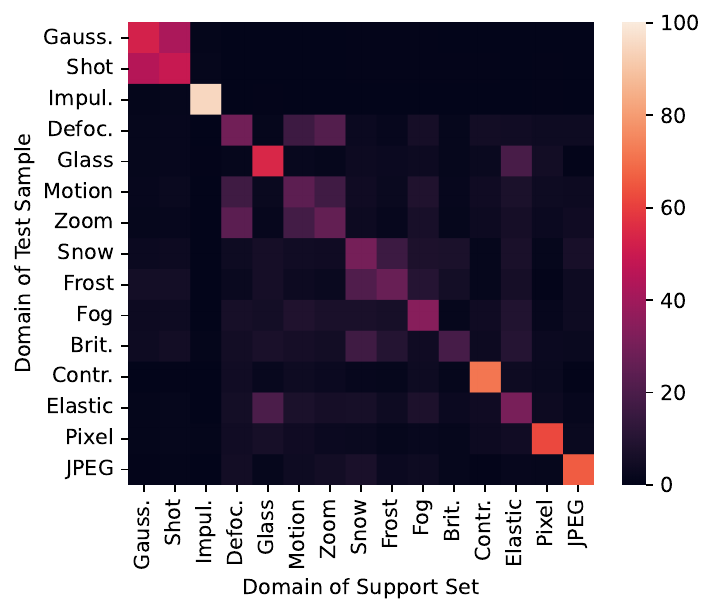}
        \subcaption{CIFAR-10-C (44.5\%)}
    \end{subfigure}
    \hfill
    \begin{subfigure}[t]{0.49\linewidth}
        \centering
        \includegraphics[width=\linewidth]{figure/vis_mat_CIFAR100C.pdf}
        \subcaption{CIFAR-100-C (40.9\%)}
    \end{subfigure}

    \begin{subfigure}[t]{0.49\linewidth}
        \centering
        \includegraphics[width=\linewidth]{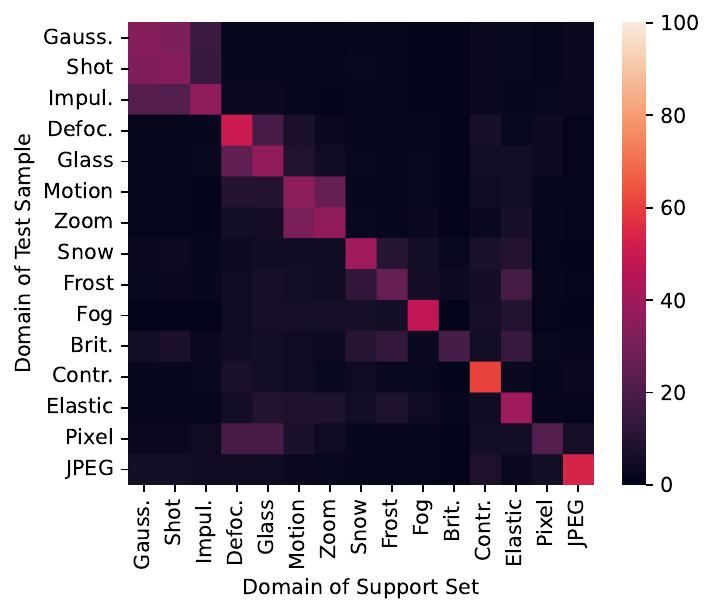}
        \subcaption{ImageNet-C (38.0\%)}
    \end{subfigure}
    \hfill
    \begin{subfigure}[t]{0.49\linewidth}
        \centering
        \includegraphics[width=\linewidth]{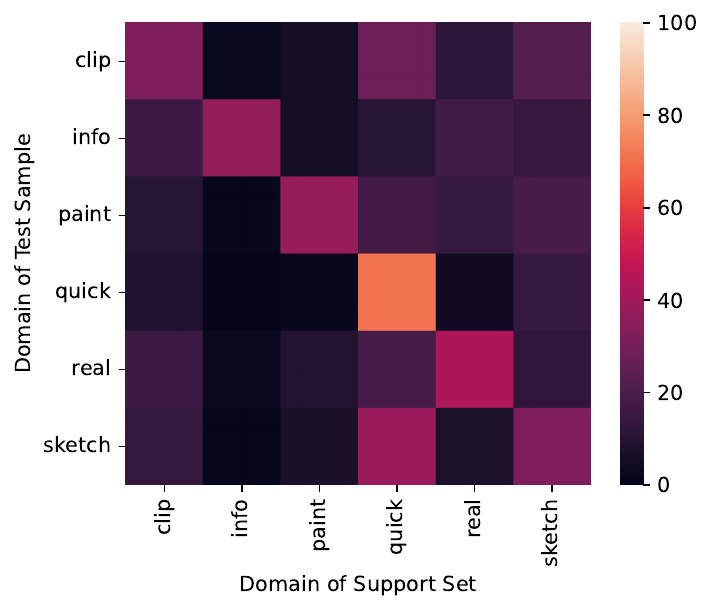}
        \subcaption{DomainNet (47.8\%)}
    \end{subfigure}
    \caption{Visualization of active sample selection.  Each entry $(i, j)$ indicates the average proportion (\%) of support samples from domain $j$ when the test sample comes from domain $i$. Numbers in the parentheses indicate the proportion of support samples that come from the same domain as the test sample.}
    \label{fig:appendix:exp:visual}
\end{figure}
\newpage
\subsection{Memory Usage (RQ3)}
\label{appendix:exp:memory}

In the main text, we report the runtime of {\algname} and baseline TTA methods. Here, we further analyze memory usage as reported in Table~\ref{tab:memory}. Although {\algname} is highly efficient in computation due to the embedding-gradient cache, this efficiency is achieved by trading memory for time: storing both embeddings and gradients inevitably consumes GPU memory. In our CIFAR-100-C experiments, the GPU memory usage of the embedding-gradient cache is
\begin{align*}
    \underbrace{100 \vphantom{(}}_{\text{Number of classes $C$}} 
    \times 
    \underbrace{750 \vphantom{(}}_{\text{Maximum Capacity per class $K$}} 
    \times 
    ( 
    \underbrace{512 \vphantom{(}}_{\text{Embedding dim}} 
    + 
    \underbrace{39{,}936 \vphantom{(}}_{\text{Gradient dim}} 
    )
    \times 
    \underbrace{2\text{ B} \vphantom{(}}_{\text{Half dtype}}
    \approx 5,785 \text{ MiB, }
\end{align*}
which constitutes the main source of GPU memory consumption. In addition, the reparameterization trick introduced in Appendix~\ref{appendix:discussion:reparam} to improve parallelism also incurs about 1,300 MiB additional memory overhead. Overall, {\algname} requires 14{,}526~MiB of GPU memory in total. While this represents an increase compared to the simplest baselines (e.g., EntMin), it remains within a reasonable range given the substantial performance and efficiency gains. 

\paragraph{Ways to save memory}
When GPU memory is highly limited, several trade-offs are available.
Reducing the maximum capacity $K$ (e.g., from 750 to 300, which only slightly affects accuracy according to Figure~\ref{fig:hp:all}) or lowering the batch size (which has negligible impact as shown in Figure~\ref{fig:hp:bs}) both effectively decrease memory usage.
Additionally, we provide an alternative configuration that offloads gradients to the CPU: embeddings are stored on the GPU for the computation of distance and aggregation weights, while gradients are stored and aggregated on the CPU before the aggregated result is transferred back to the GPU.
Although this design introduces extra CPU–GPU communication and increases computation time, the GPU memory used by the cache drops drastically to 
\begin{align*}
    \underbrace{100 \vphantom{(}}_{\text{Number of classes $C$}} 
    \times 
    \underbrace{750 \vphantom{(}}_{\text{Maximum Capacity per class $K$}} 
    \times 
    \underbrace{512 \vphantom{(}}_{\text{Embedding dim}} 
    \times 
    \underbrace{2\text{ B} \vphantom{(}}_{\text{Half dtype}}
    \approx 73 \text{ MiB, }
\end{align*}
which is nearly negligible compared to the model’s computational graph.

\begin{table}[h!]
    \centering
    \caption{Comparison of GPU Memory usage on CIFAR-100-C.}
    \label{tab:memory}
    \footnotesize
    \begin{tabular}{lrr}
        \toprule
        Method & GPU Memory & Testing Time \\
        \midrule
        EntMin                                      & 7,290 MiB     & 11m42s \\
        {\algname} (cache on CPU)                   & 8,642 MiB     & 2h13m03s \\
        {\algname} (embeddings on GPU, gradients on CPU)      & 8,724 MiB     & 46m47s \\
        {\algname} (cache on GPU)                   & 14,526 MiB    & 14m08s \\
        \bottomrule 
    \end{tabular}
\end{table}

\newpage
\subsection{Hyperparameter Sensitivity}
\label{appendix:exp:hparams}

\paragraph{Effect of hyperparameters introduced by {\algname} ($K, k, \beta$)}
{\algname} introduces three key hyperparameters: the maximum capacity per class $K$, the number of retrieved samples per class $k$, and the similarity scaler $\beta$. We evaluate their effects on all datasets, as shown in Figure \ref{fig:hp:all}.
\begin{itemize}
    \item Across all datasets, \textbf{{\algname} consistently outperforms vanilla entropy minimization (EntMin) over a wide range of hyperparameter choices}, indicating that the method is generally robust and effective once the proposed components are incorporated. 
    \item Moreover, the \textbf{influence patterns of these hyperparameters are remarkably consistent} across different datasets.
    \item On ImageNet-C, since the number of samples per class and per domain is relatively small, we set $k=1$. In this case, the similarity scaler $\beta$ becomes irrelevant, as its primary role is to assign finer-grained weights among multiple retrieved samples within each class, which is unnecessary when only one sample is retrieved.
\end{itemize}

\begin{figure}[h!]
    \centering
    \begin{subfigure}[t]{1.0\linewidth}
        \centering
        \includegraphics[width=\linewidth]{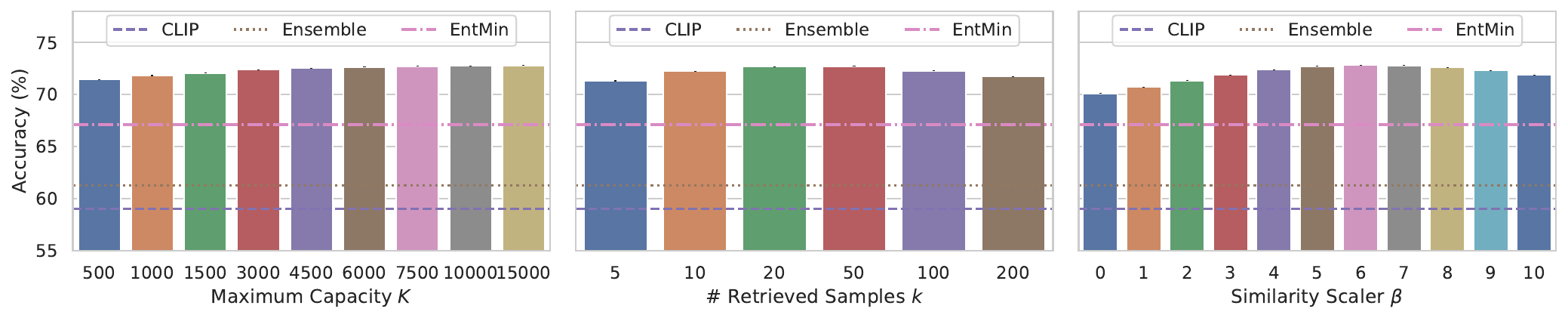}
        \subcaption{CIFAR-10-C}
    \end{subfigure}
    \hfill
    \begin{subfigure}[t]{1.0\linewidth}
        \centering
        \includegraphics[width=\linewidth]{figure/Cifar100_hp.pdf}
        \subcaption{CIFAR-100-C}
    \end{subfigure}
    
    \begin{subfigure}[t]{1.0\linewidth}
        \centering
        \includegraphics[width=\linewidth]{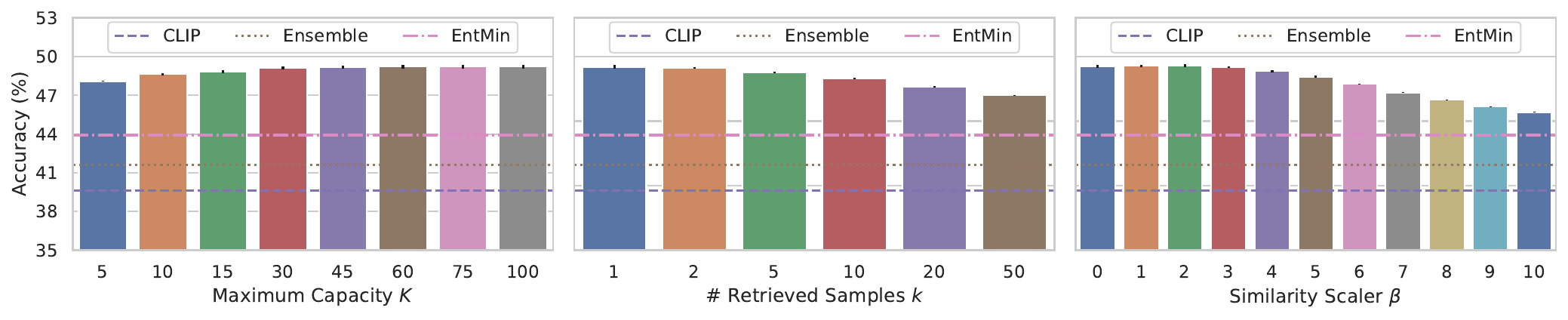}
        \subcaption{ImageNet-C}
    \end{subfigure}
    \hfill
    \begin{subfigure}[t]{1.0\linewidth}
        \centering
        \includegraphics[width=\linewidth]{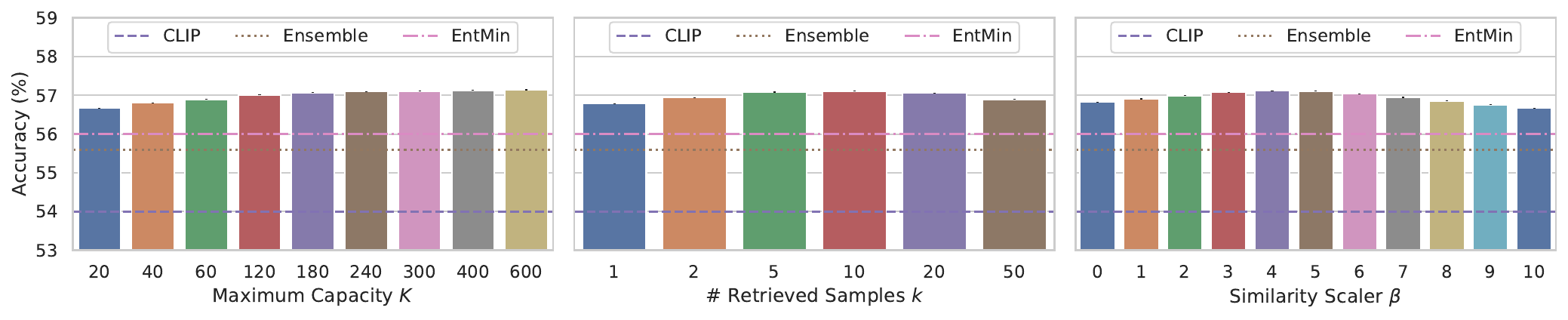}
        \subcaption{DomainNet}
    \end{subfigure}
    \caption{Hyperparameter sensitivity. (EntMin refers to entropy minimization without {\algname}.)}
    \label{fig:hp:all}
\end{figure}


\paragraph{Effect of learning rate $\eta$}
The learning rate is one of the most critical hyperparameters in test-time adaptation.
Figure~\ref{fig:hp:lr} compares the performance of {\algname} and vanilla entropy minimization (EntMin) under different learning rates. 
We observe that:
\begin{itemize}
    \item Within a broad range of learning rates, {\algname} yields noticeable performance gains compared to the non-adaptive ensemble baseline. 
    \item {\algname} consistently outperforms EntMin across all settings. 
    \item The optimal learning rate remains highly consistent across different datasets and ViT model scales, demonstrating the robustness of the proposed method.
\end{itemize}

\begin{figure}[h!]
    \centering
    \begin{subfigure}[t]{0.49\linewidth}
        \centering
        \includegraphics[width=\linewidth]{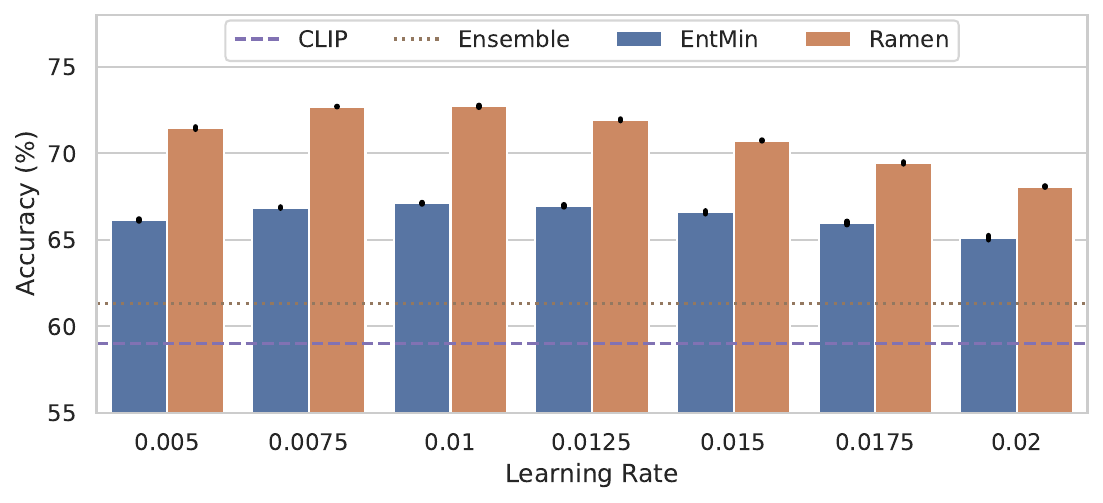}
        \subcaption{CIFAR-10-C}
    \end{subfigure}
    \hfill
    \begin{subfigure}[t]{0.49\linewidth}
        \centering
        \includegraphics[width=\linewidth]{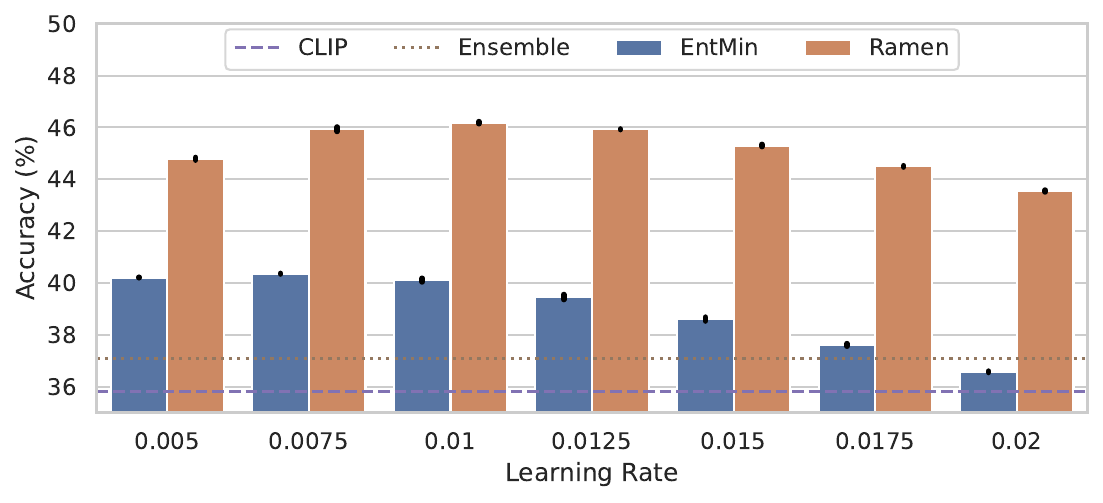}
        \subcaption{CIFAR-100-C}
    \end{subfigure}
    
    \begin{subfigure}[t]{0.49\linewidth}
        \centering
        \includegraphics[width=\linewidth]{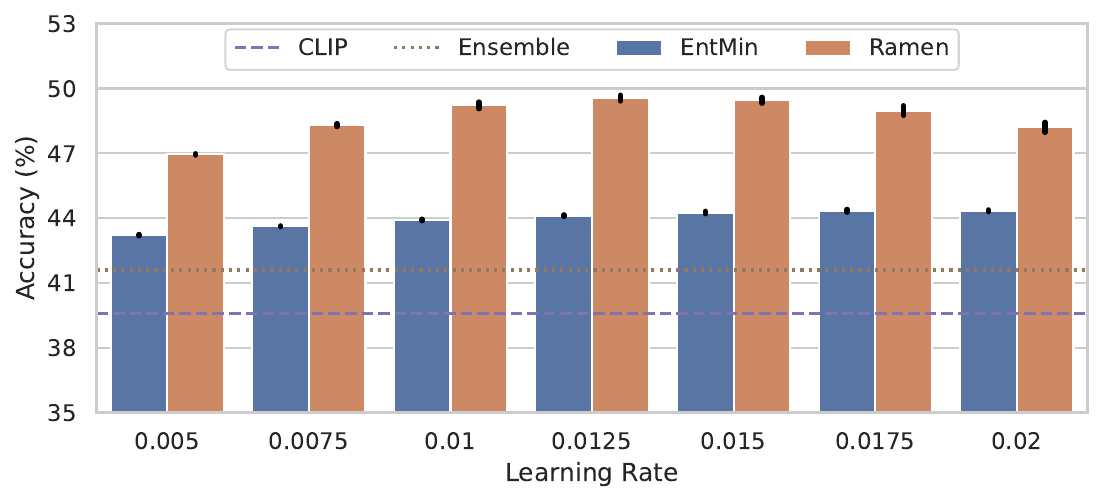}
        \subcaption{ImageNet-C}
    \end{subfigure}
    \hfill
    \begin{subfigure}[t]{0.49\linewidth}
        \centering
        \includegraphics[width=\linewidth]{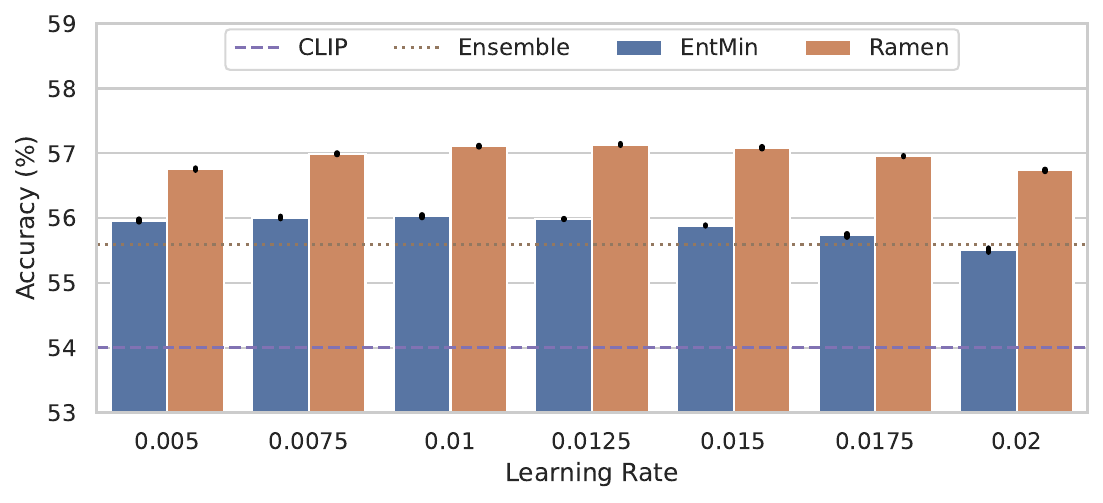}
        \subcaption{DomainNet}
    \end{subfigure}
    \caption{Effect of learning rate. (EntMin refers to entropy minimization without {\algname}.)}
    \label{fig:hp:lr}
\end{figure}


\paragraph{Effect of batch size $B$}
A byproduct of {\algname} is its insensitivity to the test-time batch size.
Since the samples used for adaptation are retrieved from the memory (a total of $C \cdot k$ entries), the update process relies primarily on the retrieved support set rather than the incoming batch itself.
As shown in Figure~\ref{fig:hp:bs}, the performance of {\algname} remains largely stable across different batch sizes.

\begin{figure}[h]
    \centering
    \includegraphics[width=0.5\linewidth]{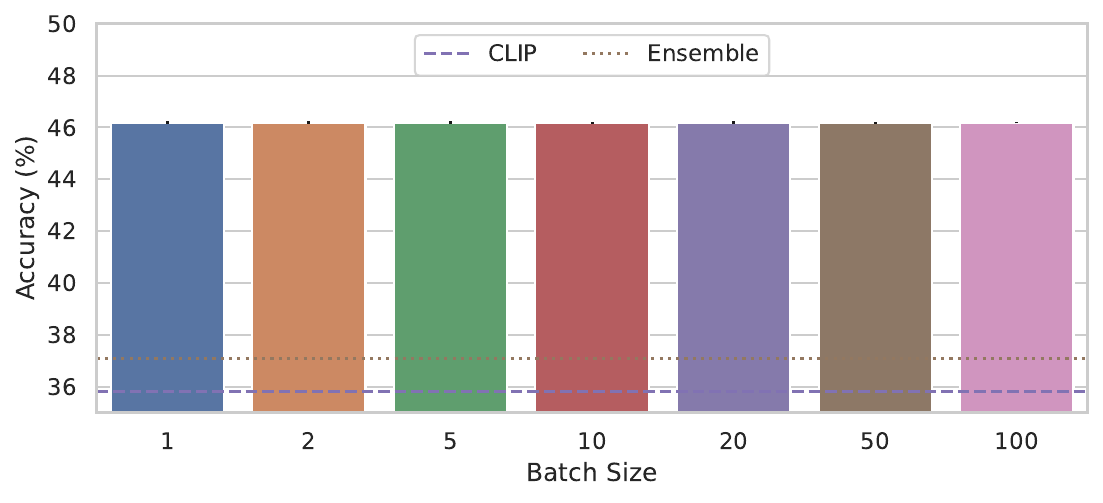}
    \caption{Effect of batch size on {\algname}.}
    \label{fig:hp:bs}
\end{figure}

\newpage
\subsection{Ablation Study}
\label{appendix:exp:ablation}

We perform an ablation study on the two components of the gradient aggregation weights, entropy and similarity, on the CIFAR-100-C dataset.
As shown in Figure~\ref{tab:ablation:weight}, removing either component degrades performance, indicating that both contribute to the effectiveness of {\algname}. 

\begin{table}[h!]
    \centering
    \caption{Ablation study on sample weighting.}
    \label{tab:ablation:weight}
    \resizebox{0.5\linewidth}{!}{
        \begin{tabular}{lccc}
            \toprule
            Method & Entropy Weighting & Similarity Weighting & Accuracy (\%) \\
            \midrule
            CLIP
            & -     & -     & 35.8 \\
            Ensemble
            & -     & -     & 37.1 \\
            EntMin
            & -     & -     & \meansd{40.1}{0.1} \\
            \midrule
            \multirow{4}{*}{{\algname}}
            & \no   & \no   & \meansd{43.9}{0.1} \\
            & \no   & \yes  & \meansd{45.2}{0.0} \\
            & \yes  & \no   & \meansd{45.2}{0.0} \\
            & \yes  & \yes  & \meansd{46.1}{0.1} \\
            \bottomrule 
        \end{tabular}
    }
\end{table}

\subsection{More Architectures}
\label{appendix:exp:architecture}

In the main text, we use ViT models of different size as the visual encoder, since they generally yield stronger performance. Here, we further evaluate {\algname} using a ResNet (RN) \cite{resnet} backbone to verify its generality. Table~\ref{tab:resnet_meansd} summarizes the results on CIFAR-100-C with RN101 as the image encoder. We observe that {\algname} remains effective, demonstrating its compatibility with different model architectures.

\renewcommand{\meansd}[2]{#1 \scriptsize(#2)}

\begin{table*}[h!]
    \centering
    \caption{Accuracy (mean (s.d.) \%) of RN101 on CIFAR-100-C under mixture of 15 corruptions. Best and second-best results are shown in \textbf{bold} and \underline{underlined}, respectively. }
    \label{tab:resnet_meansd}
    \resizebox{1.0\linewidth}{!}{
    \setlength{\tabcolsep}{3.0pt}{
        \begin{tabular}{llccccccccccccccc >{\columncolor{cvprblue!15}}c}
            \toprule
            & & \multicolumn{16}{c}{RN101 on CIFAR-100-C} \\
            \cmidrule(lr){3-18}
            Method & Venue & \multicolumn{3}{c}{Noise} & \multicolumn{4}{c}{Blur} & \multicolumn{4}{c}{Weather} & \multicolumn{4}{c}{Digital} & \multirow{2.5}{*}{Avg.} \cellcolor{white} \\
            \cmidrule(lr){3-5} \cmidrule(lr){6-9} \cmidrule(lr){10-13} \cmidrule(lr){14-17}
            & & Gauss. & Shot & Impul. 
            & Defoc. & Glass & Motion & Zoom  
            & Snow & Frost & Fog & Brit. 
            & Contr. & Elastic & Pixel & JPEG \\
            \midrule
            CLIP \cite{clip} & ICML'21 &
                10.8 & 12.9 & 8.7 & 19.3 & 12.3 & 20.8 & 24.3 & 28.5 & 31.8 & 27.5 & 36.0 & 18.4 & 19.5 & 17.7 & 20.5 & 20.6 \\
            Ensemble & - & 
                11.7 & 14.5 & 9.6 & 20.5 & 11.4 & 21.5 & 25.9 & 29.2 & 31.8 & 26.8 & 37.3 & 18.8 & 18.5 & 19.3 & 21.7 & 21.2 \\
            Tent \cite{tent} & ICLR'21 & 
                \meansd{12.2}{0.1} & \meansd{15.1}{0.1} & \meansd{10.3}{0.1} & \meansd{23.2}{0.2} & \meansd{10.3}{0.1} & \meansd{22.4}{0.1} & \meansd{28.4}{0.1} & \meansd{30.4}{0.2} & \meansd{31.6}{0.1} & \meansd{26.8}{0.1} & \meansd{39.1}{0.2} & \meansd{20.6}{0.1} & \meansd{18.2}{0.1} & \meansd{19.7}{0.2} & \meansd{23.0}{0.2} & \meansd{22.1}{0.0} \\
            NOTE \cite{note} & NeurIPS'22 & 
                \meansd{13.9}{0.5} & \meansd{17.2}{0.8} & \meansd{13.2}{0.5} & \meansd{29.9}{0.6} & \meansd{11.8}{0.7} & \meansd{25.8}{0.6} & \meansd{34.2}{0.7} & \meansd{34.0}{0.4} & \meansd{35.1}{0.7} & \meansd{30.9}{0.8} & \meansd{44.4}{0.7} & \meansd{27.2}{1.0} & \meansd{20.8}{0.7} & \meansd{20.2}{1.0} & \meansd{25.1}{0.1} & \meansd{25.6}{0.4} \\
            SAR \cite{sar} & ICLR'23 & 
                \meansd{13.0}{0.2} & \meansd{16.2}{0.1} & \meansd{11.3}{0.1} & \meansd{26.0}{0.2} & \meansd{10.2}{0.1} & \meansd{24.1}{0.2} & \meansd{31.0}{0.2} & \meansd{31.8}{0.2} & \meansd{32.8}{0.2} & \meansd{28.2}{0.1} & \meansd{41.3}{0.2} & \meansd{22.8}{0.2} & \meansd{18.8}{0.3} & \meansd{19.7}{0.2} & \meansd{23.9}{0.2} & \meansd{23.4}{0.1} \\
            RoTTA \cite{rotta} & CVPR'23 & 
                \meansd{13.3}{0.9} & \meansd{16.3}{0.8} & \meansd{12.7}{0.5} & \meansd{28.5}{0.4} & \meansd{9.6}{0.5} & \meansd{24.4}{0.5} & \meansd{30.5}{0.3} & \meansd{22.4}{0.8} & \meansd{28.2}{0.6} & \meansd{31.5}{0.5} & \meansd{32.5}{0.9} & \meansd{29.9}{0.5} & \meansd{19.5}{0.4} & \meansd{18.4}{0.8} & \meansd{22.4}{0.4} & \meansd{22.7}{0.3} \\
            TDA \cite{tda} & CVPR'24 & 
                \meansd{12.8}{0.1} & \meansd{15.6}{0.1} & \meansd{11.0}{0.2} & \meansd{20.5}{0.3} & \meansd{11.7}{0.1} & \meansd{22.5}{0.2} & \meansd{26.4}{0.2} & \meansd{30.9}{0.3} & \meansd{33.1}{0.1} & \meansd{27.7}{0.2} & \meansd{38.7}{0.3} & \meansd{19.5}{0.1} & \meansd{19.6}{0.1} & \meansd{20.2}{0.1} & \meansd{22.4}{0.2} & \meansd{22.2}{0.1} \\
            DMN-ZS \cite{dmn} & CVPR'24 & 
                \meansd{12.0}{0.0} & \meansd{14.7}{0.1} & \meansd{7.8}{0.1} & \meansd{21.0}{0.1} & \meansd{11.7}{0.1} & \meansd{22.1}{0.0} & \meansd{26.4}{0.1} & \meansd{30.0}{0.1} & \meansd{32.4}{0.1} & \meansd{27.1}{0.0} & \meansd{37.9}{0.1} & \meansd{19.3}{0.1} & \meansd{19.0}{0.0} & \meansd{19.6}{0.1} & \meansd{22.4}{0.1} & \meansd{21.6}{0.0} \\
            WATT-S \cite{watt} & NeurIPS'24 & 
                \meansd{\underline{15.8}}{0.1} & \meansd{\underline{18.9}}{0.2} & \meansd{\underline{15.4}}{0.1} & \meansd{\underline{33.9}}{0.2} & \meansd{\textbf{13.8}}{0.1} & \meansd{\underline{28.2}}{0.1} & \meansd{\underline{37.9}}{0.1} & \meansd{\underline{36.6}}{0.3} & \meansd{\textbf{38.2}}{0.2} & \meansd{\underline{34.6}}{0.2} & \meansd{\underline{47.9}}{0.2} & \meansd{\underline{31.6}}{0.3} & \meansd{\textbf{23.7}}{0.2} & \meansd{20.5}{0.2} & \meansd{\underline{26.6}}{0.1} & \meansd{\underline{28.2}}{0.1} \\
            CLIPArTT \cite{clipartt} & WACV'25 & 
                \meansd{11.0}{0.1} & \meansd{13.4}{0.1} & \meansd{10.4}{0.1} & \meansd{26.7}{0.2} & \meansd{12.4}{0.2} & \meansd{23.1}{0.2} & \meansd{30.5}{0.4} & \meansd{31.3}{0.2} & \meansd{32.6}{0.1} & \meansd{28.6}{0.1} & \meansd{40.8}{0.3} & \meansd{24.2}{0.2} & \meansd{19.1}{0.2} & \meansd{\underline{21.8}}{0.2} & \meansd{23.4}{0.3} & \meansd{23.3}{0.1} \\
            Mint \cite{mint} & NeurIPS'25 & 
                \meansd{15.3}{0.1} & \meansd{18.6}{0.2} & \meansd{13.4}{0.1} & \meansd{31.2}{0.0} & \meansd{11.1}{0.1} & \meansd{26.5}{0.1} & \meansd{35.7}{0.1} & \meansd{34.9}{0.1} & \meansd{35.1}{0.1} & \meansd{31.5}{0.1} & \meansd{44.9}{0.1} & \meansd{27.9}{0.0} & \meansd{21.7}{0.0} & \meansd{18.5}{0.1} & \meansd{24.8}{0.1} & \meansd{26.1}{0.0} \\
            {\algname} & - & 
                \meansd{\textbf{19.0}}{0.2} & \meansd{\textbf{21.9}}{0.2} & \meansd{\textbf{16.3}}{0.2} & \meansd{\textbf{36.3}}{0.4} & \meansd{\underline{13.1}}{0.1} & \meansd{\textbf{30.4}}{0.2} & \meansd{\textbf{40.3}}{0.2} & \meansd{\textbf{36.7}}{0.2} & \meansd{\underline{36.8}}{0.2} & \meansd{\textbf{35.5}}{0.2} & \meansd{\textbf{49.3}}{0.3} & \meansd{\textbf{37.7}}{0.3} & \meansd{\underline{22.3}}{0.4} & \meansd{\textbf{25.9}}{0.4} & \meansd{\textbf{26.9}}{0.2} & \meansd{\textbf{29.9}}{0.1} \\
            \bottomrule 
        \end{tabular}
    }
    }
\end{table*}


\end{document}